\documentclass{article}

\PassOptionsToPackage{numbers, compress}{natbib}

 \usepackage[main, final]{neurips_2026}

\usepackage[utf8]{inputenc}
\usepackage[T1]{fontenc}
\usepackage{float}
\usepackage{hyperref}
\usepackage{url}
\usepackage{booktabs}
\usepackage{amsfonts}
\usepackage{nicefrac}
\usepackage{microtype}
\usepackage[table]{xcolor}

\usepackage{etoolbox}
\usepackage{capt-of}
\usepackage{needspace}

\usepackage{amsmath,amsfonts,bm}

\def\eqref#1{equation~\ref{#1}}

\def\1{\bm{1}}

\DeclareMathAlphabet{\mathsfit}{\encodingdefault}{\sfdefault}{m}{sl}
\SetMathAlphabet{\mathsfit}{bold}{\encodingdefault}{\sfdefault}{bx}{n}

\DeclareMathOperator*{\argmin}{arg\,min}

\usepackage{amsmath,amsthm,amssymb}
\usepackage{thmtools}
\usepackage{enumerate}
\usepackage{enumitem}
\usepackage{bm}
\usepackage{adjustbox}
\usepackage{multicol}
\usepackage{multirow}
\usepackage{subcaption}
\AtBeginEnvironment{table}{\color{black}\hypersetup{linkcolor=black,citecolor=black,urlcolor=black}}
\AtBeginEnvironment{table*}{\color{black}\hypersetup{linkcolor=black,citecolor=black,urlcolor=black}}
\AtBeginEnvironment{figure}{\color{black}}
\AtBeginEnvironment{figure*}{\color{black}}
\usepackage[capitalize,noabbrev]{cleveref}
\Crefname{appsec}{Appendix}{Appendices}
\crefname{appsec}{Appendix}{Appendices}
\Crefname{appsubsec}{Appendix}{Appendices}
\crefname{appsubsec}{Appendix}{Appendices}
\usepackage{algorithm}
\usepackage{algpseudocode}
\usepackage{wrapfig}
\usepackage{pifont}
\newcommand{\cmark}{\ding{51}}%
\newcommand{\xmark}{\ding{55}}%
\definecolor{RowHighlight}{gray}{0.95}

\newtheorem{theorem}{Theorem}[section]

\newtheorem{lemma}[theorem]{Lemma}
\newtheorem{proposition}[theorem]{Proposition}
\newtheorem{remark}[theorem]{Remark}

\newtheorem{corollary}[theorem]{Corollary}
\newtheorem{definition}{Definition}[section]

\newtheorem{assumption}{Assumption}[section]

\definecolor{GoogleRed}{RGB}{234, 67, 53}
\definecolor{GoogleBlue}{RGB}{66, 133, 244}
\definecolor{GoogleGreen}{RGB}{52, 168, 83}
\definecolor{ourlightblue}{RGB}{232, 240, 254}
\definecolor{KWRed}{RGB}{134, 67, 53}
\definecolor{KWBlue}{RGB}{0, 133, 200}
\definecolor{KWGreen}{RGB}{20, 152, 102}
\definecolor{KWlightblue}{RGB}{232, 240, 254}

\hypersetup{
colorlinks = true,
urlcolor = KWRed,
linkcolor = KWGreen,
citecolor = KWBlue,
}

\title{Deep Barycentric Regression for \\Optimal Transport Map Estimation and \\Its Statistical Optimality}

\author{
  Kunwoong Kim\\
  KAIST AI\\
  \texttt{kunwoong.kim@kaist.ac.kr}\\
  \And
  Insung Kong\\
  UNIST\\
  \texttt{kong@unist.ac.kr}\\
  \And
  Yongdai Kim\thanks{Corresponding author.}\\
  Seoul National University\\
  \texttt{ydkim0903@gmail.com}\\
}

\begin{document}
\frenchspacing

\maketitle

\begin{abstract}
    The optimal transport (OT) map provides a geometric transformation for aligning probability distributions and has become a useful tool in machine learning.
    However, existing estimators of the OT map still exhibit a gap between sharp statistical guarantees and practical parametric estimation based on stable training objectives.
    Theoretical estimators achieve minimax optimal convergence rates, but they are typically nonparametric and can incur demanding implementation design or inference costs.
    Practical estimators are parametric and scalable, but their statistical guarantees remain underexplored, and their min-max, adversarial-like training objectives can be sensitive to optimization algorithms.
    We propose BROT (Barycentric Regression for Optimal Transport)\renewcommand{\thefootnote}{\textcolor{black}{$\diamondsuit$}}\renewcommand{\theHfootnote}{code}\footnote{Code at \url{https://github.com/kwkimonline/BROT}.}, a simple two-step method that first computes the unregularized OT plan and then fits a deep neural network (DNN) to the induced barycentric targets by least-squares regression.
    Under standard regularity conditions, we prove that the DNN estimator of BROT attains the minimax convergence rate, when the ground-truth OT map is Lipschitz.
    Numerical studies on synthetic datasets and an image dataset show that BROT provides accurate map estimates, strong target distribution matching, and competitive transport costs, compared to existing estimation methods.
    Experiments on two downstream tasks, single-cell perturbation prediction and unsupervised domain adaptation, further suggest that the accurate estimation of BROT can translate into stronger task performance.
\end{abstract}
\setcounter{footnote}{0}

\section{Introduction}\label{sec:intro}

Optimal transport (OT) provides a geometric framework for comparing and aligning probability distributions.
Given two distributions and a cost function, the OT map defines a way to transport masses from one distribution (called the source) to the other (called the target) with the minimum expected cost.
As a result, the OT map not only quantifies a distributional discrepancy but also yields an explicit transportation, and has therefore been
widely used in applications such as domain adaptation, self-supervised learning, generative modeling, algorithmic fairness, covariate balancing, point cloud matching, and image translation, to name just a few \cite{peyre2020computationaloptimaltransport,caron2020unsupervised,10.5555/3305381.3305404,pmlr-v89-redko19a,pmlr-v97-gordaliza19a,lee2025unpaired,blanchet2024automatic,fan2023neural,korotin2023neural,kim2025fairness}.

\paragraph{Related work.}

Recent theoretical works on estimation of the OT map have focused on recovering the ground-truth (population) OT map from i.i.d. data sampled from the source and target distributions.
Representative examples of statistically minimax optimal estimators include the wavelet-based estimator of \cite{hutter2020minimaxestimationsmoothoptimal}, which estimates the potential, whose gradient gives the OT map, and the $1$-nearest neighbor (1NN) estimator of \cite{manole2024pluginestimationsmoothoptimal}, which extends the OT plan (the optimal joint mass assignment between source and target data), to unseen data by nearest neighbor interpolation.
More broadly, \cite{Deb2021RatesOE} analyzed plug-in estimators based on smoothed densities, \cite{pooladian2023minimaxestimationdiscontinuousoptimal} obtained a minimax optimal estimator restricted to the semi-discrete setting (continuous source, discrete target), and \cite{10.1214/24-AOS2482} studied OT map estimation over general classes of convex functions without a practical training algorithm.
\cite{vacher2024optimal} also proposed a near-optimal estimator for smooth OT maps, but with worst-case computational cost growing as the fifth power of the sample size.
These works provide meaningful statistical guarantees, but the resulting estimators are mainly designed for theoretical analysis rather than practical use.
In particular, they are highly nonparametric and typically require density smoothing, basis construction, nearest neighbor search, or regularization tuning, making their implementation demanding (often nearly infeasible) and their inference costs high.
\cref{sec:appen-complexity} compares the (theoretical) computational complexity, while \cref{tab:simuls_2} reports the actual training and inference times measured in our experiments.
\cref{sec:baselines_details} further covers more details about these methods.

On the algorithmic side, practical OT methods based on regularizations \cite{NIPS2013_af21d0c9,seguy2018large} and DNN parameterizations \cite{fan2023neural,korotin2023neural,pmlr-v119-makkuva20a,rout2022generative,choi2025overcoming,choi2025improving} have provided several ways to estimate OT maps or plans.
The Sinkhorn algorithm \cite{NIPS2013_af21d0c9} computes an entropically regularized transport plan between two empirical distributions, but it does not by itself provide a parametric transport map.
To obtain a map, \cite{seguy2018large} proposed a two-step procedure that first finds a regularized OT plan by stochastic optimization and then trains a DNN to fit the resulting regularized barycentric targets.
More recent parametric (DNN-based) estimators \cite{fan2023neural,korotin2023neural,pmlr-v119-makkuva20a,rout2022generative,choi2025overcoming,choi2025improving} parameterize transport maps, but they typically rely on min-max, adversarial-like learning objectives, which can be sensitive to optimization algorithms.
\cite{tarasov2026statistical} derived a generalization bound for the min-max learning OT estimators, but their sharp finite-sample or minimax optimality is not known.
\cref{sec:baselines_details} discusses more details about these methods.

\paragraph{Our goal and proposed approach.}

Bridging the gap between statistical optimality and practical feasibility raises a natural question: \textit{Can we construct a parametric estimator of the OT map that is both statistically optimal and practically feasible?}
In this paper, we aim to resolve this gap by proposing a learning algorithm of the OT map based on DNNs that attains statistical optimality (i.e., the minimax convergence rate) while remaining practically trainable by standard regression.
The proposed algorithm, BROT (\textit{Barycentric Regression for Optimal Transport}), proceeds in two steps.
First, we solve the unregularized empirical Kantorovich problem by a standard linear program to obtain the OT plan between the source and target empirical distributions.
We then use this plan to form barycentric targets for the source data.
Second, we train a Lipschitz-constrained DNN to fit these targets by a standard least-squares regression.
In other words, the barycentric targets serve as output variables for the source data, thereby reducing OT map estimation to a standard regression problem.
\cref{fig:diagram} depicts the overall framework of BROT.

\begin{figure}[h!]
    \centering
    \includegraphics[width=0.99\linewidth]{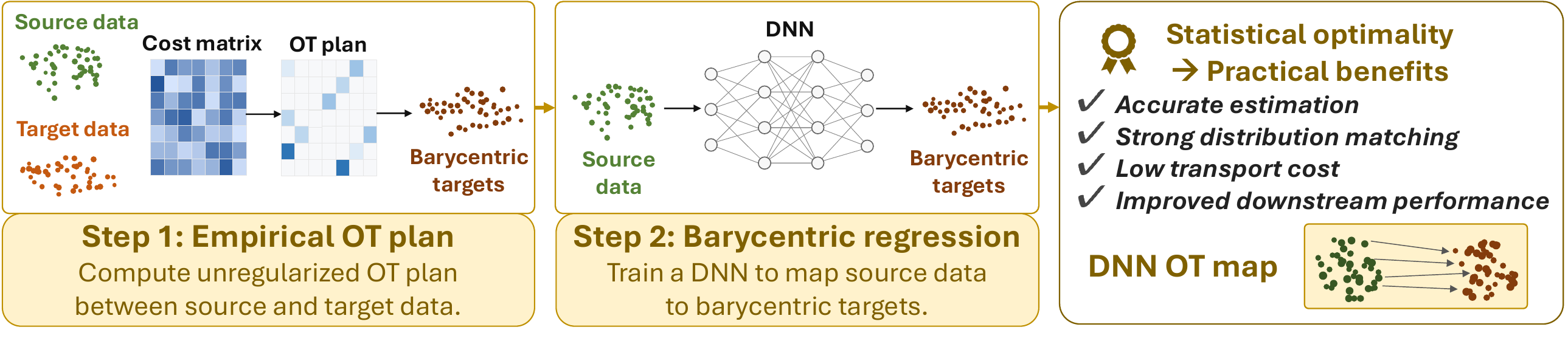}
    \caption{\textbf{Outline of our proposed framework BROT.} We first compute the unregularized OT plan, and train a DNN to fit source data to barycentric targets obtained from the OT plan.}
    \label{fig:diagram}
\end{figure}

Our main contributions are summarized as follows:
\begin{itemize}[topsep=0pt, leftmargin=1.5em, labelsep=0.5em]
    \item[$\diamond$] \textbf{Statistically optimal and practically feasible OT map estimation:}
    We propose BROT, a barycentric regression algorithm for estimating the OT map using DNNs based on the unregularized OT plan.
    Under standard regularity conditions, we prove that the DNN estimator of BROT attains the minimax convergence rate when the ground-truth OT map is Lipschitz, while it remains trainable by a standard least-squares regression without min-max optimization.
    To our knowledge, the DNN estimator of BROT is the first parametric estimator that is simultaneously minimax optimal and practically feasible.

    \item[$\diamond$] \textbf{Empirical performance in statistical convergence, map quality, and downstream utility:}
    We numerically verify that the estimation error (i.e., the difference between the trained map and the ground-truth OT map) of BROT shows a near-minimax convergence trend, and is lower than the estimation errors of the existing baseline methods.
    On synthetic and high-dimensional image datasets, BROT attains the lowest Wasserstein distance between the transported source and target distributions (i.e., produces a valid transport map), while maintaining a comparable transport cost compared with the baselines.
    On two downstream applications, single-cell perturbation prediction and unsupervised domain adaptation, BROT achieves the best or competitive performance compared to the baselines.
\end{itemize}

The remainder of this paper is organized as follows.
\cref{sec:preliminary} introduces preliminaries on the OT map and OT plan, and \cref{sec:DNN_algorithm} presents our proposed algorithm, BROT.
\cref{sec:minimax} states the main theoretical result, showing that under standard regularity conditions, the DNN estimator of BROT attains the minimax convergence rate for Lipschitz ground-truth OT map.
\cref{sec:exp_cdot} reports the results of three sets of experiments showing practical benefits of BROT:
(i) superior finite-sample performance for estimating the OT map,
(ii) validity of the trained map as a transport map (i.e., whether the transported source distribution is close to the target distribution) with low transport cost,
and (iii) improved downstream task performance on two real-world applications.
\cref{sec:discussion} concludes with a summary of our findings and directions for future work.

\section{Preliminaries}\label{sec:preliminary}

\paragraph{Optimal transport map and its relaxation.}

Let $P$ and $Q$ denote the source and target probability distributions, respectively, supported on $\Omega \subseteq \mathbb{R}^{d}$.
Suppose we observe data $X_1,\dots,X_n \stackrel{\textup{i.i.d.}}{\sim} P$ and $Y_1,\dots,Y_m \stackrel{\textup{i.i.d.}}{\sim} Q.$
The corresponding empirical source and target distributions are defined as
$ P_n (\cdot) := \frac{1}{n}\sum_{i=1}^n \delta_{X_i} (\cdot) $
and
$ Q_m (\cdot) := \frac{1}{m}\sum_{j=1}^m \delta_{Y_j} (\cdot), $
respectively.
For a (measurable) map $\mathbf{T}:\Omega\to\Omega$, the push-forward measure of $P$ induced by $\mathbf{T}$ is defined by $ \mathbf{T}_{\#}P(\cdot)=P(\mathbf{T}^{-1}(\cdot)). $
That is, $\mathbf{T}_{\#}P$ is the distribution of $\mathbf{T}(X)$ when $X \sim P.$
We say $\mathbf{T}$ is a \textit{transport map} from $P$ to $Q$ when it satisfies $\mathbf{T}_{\#}P=Q.$
The \textit{optimal transport (OT) map} $\mathbf{T}_0$ is a transport map that minimizes the expected transport cost\footnote{We use the squared Euclidean cost to ensure the existence and uniqueness of the OT maps, as in \cite{peyre2020computationaloptimaltransport,korotin2023neural,hutter2020minimaxestimationsmoothoptimal,manole2024pluginestimationsmoothoptimal,seguy2018large}.} among all transport maps:
\begin{equation}\label{def:otm-eq1}
    \mathbf{T}_{0} := \argmin_{\mathbf{T}:\mathbf{T}_{\#}P=Q} \int_{\Omega} \| x - \mathbf{T}(x) \|^2 dP(x).
\end{equation}
\cref{def:otm-eq1} is the so-called Monge problem \cite{monge1781memoire}.
Under mild regularity conditions, such as absolute continuity of $P$ and finite second moments, Brenier's theorem guarantees the existence and uniqueness of $\mathbf{T}_0$ \cite{hutter2020minimaxestimationsmoothoptimal,manole2024pluginestimationsmoothoptimal,villani2003topics,villani2008optimal,gunsilius_2022}.
See \cref{sec:appen-dual_detail} for further theoretical backgrounds on the OT map.

Kantorovich relaxed the problem by replacing the deterministic map with a \emph{joint distribution} on $\Omega \times \Omega$ whose marginals are $P$ and $Q$ \cite{villani2003topics,villani2008optimal,Kantorovich2006OnTT}.
The Kantorovich problem is formulated as
\begin{equation}\label{eq:kantorovich}
    \inf_{\Gamma \in \Pi(P, Q)} \mathbb{E}_{(X, Y) \sim \Gamma} \| X - Y \|^{2}
    =
    \inf_{\Gamma \in \Pi(P, Q)}
    \int_{\Omega \times \Omega} \|x - y\|^{2}   d\Gamma(x, y),
\end{equation}
where $\Pi(P, Q)$ denotes the set of all joint distributions on $\Omega \times \Omega$ whose marginals are $P$ and $Q$.
Note that the resulting optimal transport cost defines the squared $2$-Wasserstein distance, that is, $W_{2}(P,Q)^{2} := \inf_{\Gamma \in \Pi(P,Q)} \int_{\Omega \times \Omega} \|x-y\|^{2}\,d\Gamma(x,y)$.

\paragraph{Optimal transport plan.}

If $P (\cdot) = \delta_{0} (\cdot)$ and $Q (\cdot) = \tfrac{1}{2} \delta_{-1} (\cdot) + \tfrac{1}{2} \delta_{1} (\cdot),$ then the single point mass at $X = 0$ must be split between $Y = -1$ and $Y = 1$, so no deterministic map satisfies $\mathbf{T}_{\#}P = Q,$ and the Monge problem in \cref{def:otm-eq1} is infeasible.
The Kantorovich problem in \cref{eq:kantorovich}, by contrast, has a unique optimal joint distribution that places mass $\frac{1}{2}$ on each of $(0,-1)$ and $(0,1).$

When $P$ and $Q$ are discrete, such as the empirical measures $P_n$ and $Q_m$, the Kantorovich problem reduces to a linear program \cite{peyre2020computationaloptimaltransport,villani2008optimal,flamary2021pot} (see \cref{sec:intro_cdot} for its formulation).
We refer to any element of $\Pi(P_n,Q_m)$ as a \textit{transport plan}, and to the solution of the Kantorovich problem as an \textit{optimal transport (OT) plan}.
Unlike the OT plan, the OT map is generally not well-defined for discrete measures.
Thus, estimating an OT map from empirical measures requires an out-of-sample extension, e.g., an approach for constructing a map from the OT plan and generalizing it to unseen data.
In \cref{sec:DNN_algorithm}, we present our extension from the OT plan to a parametric map.

\section{BROT: Barycentric Regression for Optimal Transport}\label{sec:DNN_algorithm}

This section introduces our proposed OT map estimation algorithm, BROT (\textit{Barycentric Regression for Optimal Transport}).
The idea introduced in \cref{sec:motiv} is to reduce the OT map estimation to a regression problem where outputs are the \emph{barycentric targets}, i.e., the weighted averages of target data induced by the (unregularized) OT plan.
The detailed algorithm is formalized in \cref{sec:intro_cdot}, which consists of two steps:
(i) computing the OT plan between the observed source and target data, and
(ii) training a DNN by predicting the barycentric targets produced by the first step.

\subsection{Motivation and Idea}\label{sec:motiv}

The OT plan is more than just a joint distribution or coupling between two empirical distributions.
It naturally defines barycentric targets for the source data, which in turn induce the empirical barycentric map
\cite{manole2024pluginestimationsmoothoptimal,Deb2021RatesOE,pooladian2023minimaxestimationdiscontinuousoptimal,seguy2018large}.
Here, the barycentric target of a source data point is the weighted average of the target data with weights determined by the OT plan, and the empirical barycentric map is the map that sends each source data point to its barycentric target (see \cref{sec:intro_cdot} for the definitions).

The central question is how to use these barycentric targets when constructing a parametric estimator of the OT map.
The barycentric targets and the associated empirical barycentric map provide useful information for the OT map estimation.
For example, nearest neighbor interpolation of the targets is known to attain the minimax optimal rate when the ground-truth OT map is Lipschitz \cite{manole2024pluginestimationsmoothoptimal}.
However, this nonparametric extension is inconvenient at inference time, since it must store the entire training dataset to perform a nearest neighbor search for each test point (see \cref{tab:simuls_2} in \cref{sec:appen-simulation-full} for empirical evidence).
\cite{seguy2018large} also considered fitting a DNN to barycentric targets, but their targets are obtained from a regularized OT plan.
Their main contribution is a scalable algorithm for computing the regularized OT plan between two empirical distributions, rather than a sharp statistical analysis of the resulting map estimator.
In particular, it is not known whether the DNN estimator of \cite{seguy2018large} is minimax optimal.
Further, \cite{perrot2016mapping} proposed to jointly estimate the coupling and the map by alternating optimization, but did not establish the minimax optimality.
The differences from BROT are discussed in \cref{rmk:brot_vs_lsot,rmk:brot_vs_perrot} in Appendix~\ref{sec:baselines_details}.

Our goal is to estimate the OT map by a DNN with statistical optimality.
For this purpose, we treat the unregularized barycentric targets as regression outputs for the source data, thereby reducing OT map estimation to a standard regression problem.
We prove that the resulting least-squares estimator, implemented with carefully designed DNNs, attains the statistical minimax optimality (\cref{sec:minimax}).

\subsection{Proposed Estimation Method}\label{sec:intro_cdot}

\paragraph{Step 1: Compute the OT plan.}

Recall that we are given source data $\{X_i\}_{i=1}^n$ drawn from $P$ and target data $\{Y_j\}_{j=1}^m$ drawn from $Q$.
Let $c_{ij}:=\|X_i-Y_j\|^2$ for $i \in [n]$ and $j \in [m]$.
Following \cref{eq:kantorovich}, the OT plan is defined by
\begin{equation}\label{eq:kan_obj}
    \widehat{\Gamma} := [\widehat{\gamma}_{ij}]_{i, j}
    =
    \argmin_{[\gamma_{ij}]_{i, j} \in \mathcal{M}_{nm}} \sum_{i=1}^{n}\sum_{j=1}^{m} c_{ij}\gamma_{ij},
\end{equation}
where
$ \mathcal{M}_{nm} := \{ [\gamma_{ij}]_{i, j} \in \mathbb{R}_{+}^{n\times m} : \sum_{i=1}^{n}\gamma_{ij}=1/m, \sum_{j=1}^{m}\gamma_{ij}=1/n \} $
is the set of all transport plans whose marginals are $P_n$ and $Q_m$.
This OT plan $\widehat{\Gamma}$ derives the \textit{barycentric targets}
$$
\widetilde{Y}_i := \sum_{j=1}^{m} n\widehat{\gamma}_{ij}Y_j,
\quad i \in [n],
$$
which can be viewed as outputs (response variables) for the source data $\{X_i\}_{i=1}^n$.

\paragraph{Step 2: Fit a DNN by least-squares regression.}

Let $\Phi$ be a given function class.
We define the estimator as a solution to the following least-squares regression problem:
\begin{equation}\label{def:real_est}
    \widehat{\mathbf{T}}_{nm}^{\Phi}
    :=
    \argmin_{\mathbf{T}\in\Phi}
    \frac{1}{n}\sum_{i=1}^{n}
    \|\widetilde{Y}_i-\mathbf{T}(X_i)\|^2.
\end{equation}
That is, $\mathbf{T}$ is trained on the source data $X_i$, $i\in[n]$, to fit the barycentric targets $\widetilde{Y}_i$, $i\in[n]$.

We specify $\Phi$ in \cref{def:real_est} as a class of Lipschitz-constrained DNNs, motivated by both theory and practice.
Theoretically, DNNs are known to approximate smooth functions effectively \cite{HORNIK1989359,YAROTSKY2017103,Schmidt-Hieber20201875} and practically have been used for high-dimensional data and modern machine learning pipelines \cite{lecun2015deep,Goodfellow-et-al-2016}, e.g., WGAN with Lipschitz constraint \cite{NIPS2017_892c3b1c}.
Let $\Phi^{\textup{DNN}}(L,W;\sigma,d_{\textup{out}})$ denote the class of DNNs with depth at most $L$, width at most $W$, activation function $\sigma$, and output dimension $d_{\textup{out}}$ (the full mathematical definition is deferred to \cref{sec:appen-dnn}).
For a constant $\lambda>0$, we consider the Lipschitz-constrained DNN class
$ \Phi^{\textup{DNN}}_{\lambda} := \left\{ f \in \Phi^{\textup{DNN}}(L,W;\sigma,d_{\textup{out}}) : \operatorname{Lip}(f)\le \lambda \right\}, $
where
$ \operatorname{Lip}(f) := \sup_{\substack{x \neq x' \in \Omega}} \frac{\|f(x)-f(x')\|}{\|x-x'\|}. $
We define the DNN estimator of BROT as
$ \widehat{\mathbf{T}}_{nm}^{\textup{DNN}} := \widehat{\mathbf{T}}_{nm}^{\Phi_{\lambda}^{\textup{DNN}}}.$

The Lipschitz constraint is introduced mainly for theoretical purposes, as it allows us to derive the minimax optimal convergence rate for the Lipschitz OT map in \cref{sec:minimax}.
In practice, the Lipschitz constraint can be approximately enforced in several ways:
(i) by adding the Jacobian penalty
$ \frac{1}{n}\sum_{i=1}^{n} \bigl\|\nabla \mathbf{T}(X_i)\bigr\|_{F}^{2} $
to the objective in \cref{def:real_est}, where $\Vert \cdot \Vert_{F}$ denotes the Frobenius norm \cite{NIPS2017_892c3b1c},
(ii) by constraining the spectral norms of the weight matrices \cite{miyato2018spectral},
or
(iii) by using architectures specially designed to be Lipschitz \cite{pmlr-v70-cisse17a}.
In our experiments in \cref{sec:exp_cdot}, we use the Jacobian penalty with a weight $\alpha>0$, using Gaussian probes to estimate its squared Frobenius norm.
We further examine the effect of the Jacobian penalty under input perturbations in \cref{sec:appen-robustness}, showing the robustness of BROT.
Algorithm \ref{alg:mlp-ot} summarizes the practical algorithm for BROT.

Note that, for large-scale experiments, we use a chunked approximation of Step 1, where we partition the source and target data into paired chunks, compute an (unregularized) OT plan within each chunk pair, and collect the resulting barycentric targets (see \cref{sec:appen-simulation-tuning,sec:appen-complexity} for details).

\begin{algorithm}[t]
    \small
    \caption{BROT (Barycentric Regression for Optimal Transport)}
    \label{alg:mlp-ot}
    \begin{algorithmic}[1]
    \State \textbf{Input:}
    Source data $\{X_i\}_{i=1}^n$,
    Target data $\{Y_j\}_{j=1}^m$,
    DNN function class (architecture) $\Phi$,
    Training hyperparameters
    (optimizer $\textup{Opt}$, step size $\eta$, epochs $T$, mini-batch size $B$, Jacobian-penalty weight $\alpha > 0$).
    \State \textbf{(Step 1: Compute the empirical transport plan)}
    \State Build the cost matrix $\mathbf{C} = [c_{ij}]_{i \in [n], j \in [m]} \in \mathbb{R}^{n\times m}$ where $c_{ij} = \| X_i - Y_j \|^2$.
    \State Compute
    $ \displaystyle \widehat{\Gamma} = \argmin_{\Gamma \in \mathcal{M}_{nm}} \sum_{i=1}^{n}\sum_{j=1}^{m} c_{ij}\gamma_{ij}. $
    \State Build the barycentric targets
    $ \widetilde{Y}_i=\sum_{j=1}^{m} n\widehat{\gamma}_{ij} Y_j, i\in[n]. $
    \State \textbf{(Step 2: Fit a DNN by Jacobian-penalized least squares)}
    \State Initialize a model $\mathbf{T}_\theta \in \Phi$ and set $t=0$.
    \While{$t<T$ and not converged}
        \State Sample a mini-batch $\mathcal{B}\subset [n]$.
        \State Compute the regression loss
        $ \mathcal{L}_{\textup{reg}}(\theta) = \frac{1}{|\mathcal{B}|}\sum_{i \in \mathcal{B}} \|\widetilde{Y}_{i} - \mathbf{T}_\theta(X_{i})\|^2 $.
        \State Sample $v_i\stackrel{\textup{i.i.d.}}{\sim}\mathcal N(0,I_d)$ for $i\in\mathcal B$ and compute
        $
        \widehat{\mathcal R}_J(\theta)
        =
        \frac{1}{|\mathcal B|}
        \sum_{i\in\mathcal B}
        \left\|
        \nabla_x\langle v_i,\mathbf T_\theta(x)\rangle\big|_{x=X_i}
        \right\|^2.
        $
        \State Set
        $\mathcal L(\theta)=\mathcal L_{\textup{reg}}(\theta)+\alpha\widehat{\mathcal R}_J(\theta)$.
        \State Update $\theta \leftarrow \textup{Opt}\bigl(\theta,\nabla_\theta\mathcal L(\theta),\eta\bigr)$.
        \State $t \leftarrow t+1$.
    \EndWhile
    \State \textbf{Output:} trained transport map $\mathbf{T}_\theta$.
    \end{algorithmic}
\end{algorithm}

\section{Theoretical Study: Minimax Optimality}\label{sec:minimax}

An important statistical question is whether using BROT achieves the minimax optimal convergence rate for estimating the ground-truth OT map $\mathbf{T}_0$.
As done in existing studies \cite{hutter2020minimaxestimationsmoothoptimal,manole2024pluginestimationsmoothoptimal,10.1214/24-AOS2482},
we assume the following standard regularity conditions on the source and target densities and on the smoothness of the potential.
Let $\mathcal{C}^{2}(\Omega)$ denote a class of functions with controlled derivatives up to order two.
More precise definitions of the H\"older spaces $\mathcal{C}^{\beta}(\Omega), \beta > 0$ are deferred to \cref{sec:appen-notation}.

\begin{assumption}\label{assumption-1-2}

    The distributions $P$ and $Q$ are absolutely continuous on $\Omega$ with density $p$ and $q,$ respectively.
    Moreover, there exists $\gamma>0$ such that
    $ \gamma^{-1} \le p(x), q(x) \le \gamma $
    for all $x \in \Omega$.

\end{assumption}

\begin{assumption}\label{assumption-3}

    The potential function $\phi_0$ is convex and belongs to $\mathcal{C}^{2}(\Omega).$
    Furthermore, it satisfies
    $ \lambda^{-1}I_d \preceq \nabla^2 \phi_0(x) \preceq \lambda I_d, \forall x \in \Omega $
    for some $\lambda\ge 1$.
    That is, $\mathbf{T}_0=\nabla\phi_0$ is $\lambda$-Lipschitz.

\end{assumption}

Under Assumptions \ref{assumption-1-2} and \ref{assumption-3}, \cite{hutter2020minimaxestimationsmoothoptimal} proved the minimax lower bound
$\widetilde{n}^{-2/d} \vee \widetilde{n}^{-1}$ for estimating a Lipschitz OT map,
where $\widetilde{n} := n \wedge m$ and $\Vert f \Vert_{L^{2}(P)}^{2} := \int_{\Omega} \Vert f(x) \Vert^{2} dP(x)$ for $f : \Omega \to \mathbb{R}^{d}$.
Our theoretical goal is to show that our DNN estimator of BROT in \cref{sec:DNN_algorithm} attains this rate, i.e., its risk is upper-bounded by the minimax rate.

For this purpose, we specially design the DNN architecture to depend on the sample sizes $n$ and $m$ in order to control the approximation error, following the standard convention in the literature on DNN-based nonparametric estimation \cite{YAROTSKY2017103,Schmidt-Hieber20201875,suzuki2018adaptivity}.
Define
$ \Phi_{nm}^{\textup{DNN}} := \Phi_{nm,\lambda}^{\textup{DNN}}
:= \left\{ f \in \Phi^{\textup{DNN}}(L, W_{nm}; \sigma_{\textup{mix}}, d) : \operatorname{Lip}(f) \le \lambda + c' \right\}, $
for a fixed $\lambda > 0,$
where $c'=d\lambda+1$, and $\sigma_{\textup{mix}}$ denotes a mixed ReLU/ReQU activation scheme (ReLU and ReQU applied at different nodes, for technical reasons).
The depth is set as $L \asymp 1$, while the width $W_{nm}$ grows polynomially in $\widetilde{n}.$
The explicit formulation is given by \cref{lem:size_schedule_n_rate} in \cref{sec:appen-lemmas}.

Note that this sample size-dependent design as well as the Lipschitz constraint are mainly theoretical tools for proving \cref{thm:dnn_main}.
Accordingly, \cref{thm:dnn_main} provides a statistical guarantee for the exact empirical minimizer over this (theoretical) DNN class.
In the experiments in \cref{sec:exp_cdot}, we instead use a fixed MLP regularized by a Jacobian penalty, and \cref{sec:appen-sensitivity} empirically suggests that BROT performs similarly across various architectures and Jacobian penalty weights.

Denote the resulting DNN estimator over this class by
\begin{equation}\label{def:mlp_est}
    \widehat{\mathbf{T}}_{nm}^{\textup{DNN}}
    :=
    \argmin_{\mathbf{T} \in \Phi_{nm}^{\textup{DNN}}}
    \sum_{i=1}^{n}
    \bigl\| \widetilde{Y}_i - \mathbf{T}(X_i) \bigr\|^2.
\end{equation}
\cref{thm:dnn_main} below gives the convergence rate of this DNN estimator $\widehat{\mathbf{T}}_{nm}^{\textup{DNN}},$ whose full proof and technical conditions are deferred to \cref{sec:appen-full_proof} and \cref{sec:appen-notation}, respectively.

\begin{theorem}[Minimax convergence rate for the DNN estimator of BROT]\label{thm:dnn_main}
    Under Assumptions~\ref{assumption-1-2}, \ref{assumption-3}, and regularity conditions described in \cref{sec:appen-notation}, with $\Omega=[0,1]^d$,
    we have that
    $$
    \mathbb{E} \big\Vert \widehat{\mathbf{T}}_{nm}^{\textup{DNN}} - \mathbf{T}_{0} \big\Vert_{L^2(P)}^{2}
    \lesssim
    \begin{cases}
        \widetilde{n}^{-1} \log \widetilde{n}, & d=1,\\
        \widetilde{n}^{-1} (\log \widetilde{n})^2, & d=2,\\
        \widetilde{n}^{-2/d} \log \widetilde{n}, & d\ge 3.
    \end{cases}
    $$
\end{theorem}

Note that the upper bound on the risk in \cref{thm:dnn_main} matches the minimax lower bound $\widetilde{n}^{-2/d} \vee \widetilde{n}^{-1}$ for the ground-truth OT map that is Lipschitz, up to logarithmic factors.
That is, the DNN estimator of BROT is (nearly) minimax optimal.
It is also noteworthy that \cref{thm:dnn_main} does not trivially follow from standard DNN regression theory, since the barycentric targets are computed from the same empirical OT plan and hence depend jointly on all source and target data.
Our proof handles this dependence, controls the population error by the empirical error using Lipschitz continuity similarly to \cite{manole2024pluginestimationsmoothoptimal}, and constructs a DNN with the required approximation error.
In particular, the Lipschitz condition is used to extend the estimation error at the observed source data to the population error.

Further, although several existing estimators of the OT map are also known to be minimax optimal, including \cite{hutter2020minimaxestimationsmoothoptimal,manole2024pluginestimationsmoothoptimal,10.1214/24-AOS2482}, most of them are not practically applicable.
Practically feasible learning algorithms for \cite{hutter2020minimaxestimationsmoothoptimal} and \cite{10.1214/24-AOS2482} are not available, and \cite{manole2024pluginestimationsmoothoptimal} requires unnecessarily large memory and computation at inference time.
Moreover, our experiments in \cref{sec:exp_cdot} empirically show that the DNN estimator of BROT outperforms the estimator of \cite{manole2024pluginestimationsmoothoptimal} even though both are minimax optimal.
This empirical observation suggests that not only statistical optimality but also practical feasibility is important for the OT map estimation, and BROT is such an algorithm.

\section{Experiments}\label{sec:exp_cdot}

In this section, we numerically assess the advantages of BROT from three perspectives.
First, in \cref{sec:rate}, using a synthetic dataset with a known Lipschitz ground-truth OT map, we examine whether the estimation error of BROT follows the theoretical minimax rate in \cref{thm:dnn_main}, and compare with existing estimation methods.
Second, in \cref{sec:sim}, we evaluate the quality of the trained (estimated) OT map in terms of target distribution matching and transport cost, on two synthetic datasets and a high-dimensional image dataset.
Third, in \cref{sec:app_cell,sec:app_da}, we investigate whether the improved quality of the trained OT map
translates into better performance on two downstream tasks: single-cell drug-perturbation prediction and unsupervised domain adaptation.

\paragraph{Baselines.}

Throughout \cref{sec:rate,sec:sim,sec:app_cell,sec:app_da}, we compare BROT with eight existing methods for estimating the OT map: 1NN \cite{manole2024pluginestimationsmoothoptimal}, LSOT \cite{seguy2018large}, ICNN \cite{pmlr-v119-makkuva20a}, OTP \cite{choi2025overcoming}, DIOTM \cite{choi2025improving}, MM-v2 and W2 \cite{korotin2021neural}, and the entropic map estimator (ENT) \cite{pooladian2021entropic}.
For each method, we tune the hyperparameters by grid search and select the configuration that minimizes the validation metric specified in each experiment.
More implementation details for each methods are provided in \cref{sec:appen-simulation-tuning,sec:baselines_details}.

\subsection{Statistical Convergence to the Ground-Truth Optimal Transport Map}\label{sec:rate}

\paragraph{Task setup.}

We study how the estimation error (roughly, $\mathbb{E} \Vert \widehat{\mathbf{T}}_{nm}^{\textup{DNN}} - \mathbf{T}_{0} \Vert_{L^{2}(P)}^{2}$ on unseen test data) changes as the training size $n$ increases, in a setting where $\mathbf{T}_{0}$ is a known Lipschitz OT map.
We construct a synthetic dataset on $\mathbb{R}^{2}$ whose source distribution $P = \mathcal{N}(0, I_2)|_{R}$ is the standard Gaussian truncated to the disk of radius $R=3.5$.
The ground-truth OT map is $\mathbf T_0( x)=(a_1x_1+b_1x_1^3,a_2x_2+b_2x_2^3)$ for $x = [x_{1}, x_{2}]^{\top}$, with $(a_1,a_2)=(1.5,0.8)$ and $(b_1,b_2)=(0.3,0.1)$.
Its construction and derivation are given in \cref{sec:appen-convergence_nonlinear}.

\paragraph{Implementation details.}

For $n=m\in\{500,1{,}000,2{,}000,3{,}000,5{,}000,8{,}000,10{,}000\}$, we split the source and target data into an $80{:}20$ training/validation split.
For each method, the training iteration and hyperparameters are selected using the $2$-Wasserstein distance between the transported source and the target validation data.
We then evaluate the trained maps on a test dataset of size $2{,}000$, and repeat the procedure with $10$ random seeds for each method and sample size and report the mean.

\begin{figure}[t]
    \centering
    \begin{subfigure}[t]{0.49\linewidth}
        \centering
        \includegraphics[width=\linewidth]{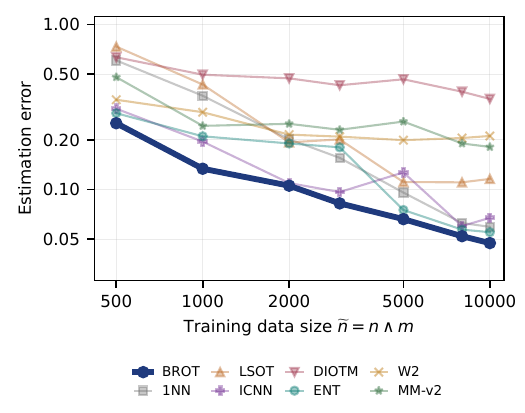}
        \caption{Training data size $\widetilde{n}$ vs. estimation error.}
        \label{fig:convergence-comparison}
    \end{subfigure}\hfill
    \begin{subfigure}[t]{0.49\linewidth}
        \centering
        \includegraphics[width=\linewidth]{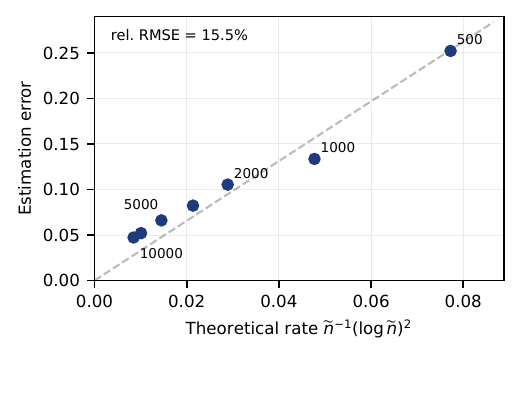}
        \caption{Theoretical rate vs. estimation error.}
        \label{fig:convergence-theory}
    \end{subfigure}
    \caption{\textbf{Statistical convergence to the ground-truth OT map.} (a) Training data size $\widetilde{n}$ vs. estimation error. Points show mean estimation errors; the corresponding standard-error bands are shown in \cref{fig:appen-convergence-sem} of \cref{sec:appen-convergence_nonlinear}. (b) Theoretical rate $\widetilde{n}^{-1}(\log\widetilde{n})^2$ vs. estimation error for BROT, with a linear regression line and its relative RMSE (rel. RMSE).}
    \label{fig:convergence_both}
    \vskip -0.15in
\end{figure}

\paragraph{Results.}

\cref{fig:convergence_both}(a) shows that the estimation error of BROT decreases monotonically with $\widetilde{n}$, and is lower than all baselines for all values of $\widetilde{n}$.
Note that we omit OTP from the panel as its estimation error is markedly larger than the other methods ($\approx 2.6$ for all $\widetilde{n}$).
Although BROT and 1NN have the same theoretical rate, BROT achieves smaller errors than 1NN at every value of $\widetilde{n}$, suggesting that BROT is empirically more accurate.
\cref{fig:convergence_both}(b) compares the estimation error of BROT with the theoretical minimax rate $\widetilde{n}^{-1}(\log \widetilde{n})^{2}.$
The relative RMSE of the fitted linear regression line is approximately $15.5\%$, indicating that the estimation error follows the theoretical minimax rate trend over the range of $\widetilde{n}$.
\cref{fig:appen-arch-robustness} of \cref{sec:appen-sensitivity} further shows that the estimation performance of BROT is not significantly sensitive to architectures and Jacobian penalty weights.

\subsection{Quality of Maps: Distribution Matching and Transport Cost}\label{sec:sim}

Beyond the estimation error for a known Lipschitz ground-truth OT map studied in \cref{sec:rate}, we assess the quality of trained maps in settings where the ground-truth OT map is unknown.
Specifically, we evaluate map quality in terms of (i) target distribution matching and (ii) transport cost: the former measures how well the learned maps match the target distribution (how much the learned maps are valid transport maps), while the latter measures the cost-efficiency of trained transport maps.

\paragraph{Task setup.}

We consider two $2$D synthetic datasets and one high-dimensional image dataset.
The two synthetic cases are Gaussian $\to$ Uniform rectangle and Gaussian $\to$ Uniform ellipse (visualizations in \cref{fig:appen-simul} in \cref{sec:appen-simulation-full}).
The image dataset is \textsc{AFHQv2} wild $\to$ cat dataset \cite{Choi_2020_CVPR} at $256{\times}256$ resolution.
Rather than in raw image pixel space, where the Euclidean distance can be a poor similarity metric, we operate in the representation space of a pretrained visual encoder.
We utilize a frozen Stable-Diffusion VAE \cite{rombach2021highresolution} as the representation encoder, which maps each image to a $4{,}096$-dimensional representation.

For the synthetic cases, we draw $n=3{,}000$ source and $m=2{,}000$ target data and split them $8{:}1{:}1$ into train/validation/test, and we use the validation data to select the best model with the lowest validation loss among all epochs.
For \textsc{AFHQv2}, we use $n=3{,}000$ wild and $m=4{,}000$ cat images with separate 500 test source and 500 test target images, all encoded into representations using the VAE encoder.
We run each method for five random trials and report the performance on the test data.

\paragraph{Implementation details.}

As the performance metrics, we consider target distribution matching by the $2$-Wasserstein distance ($\texttt{Wass}$) and the displacement by the transport cost ($\texttt{TC}$).
Lower values are better for both metrics.
For BROT, LSOT, OTP, and DIOTM, we use $4$-layer MLPs with $64$ hidden units in the synthetic cases and wider MLPs, with hidden size up to $2{,}048$, on \textsc{AFHQv2} dataset.
ICNN uses the input-convex network architecture of \cite{pmlr-v119-makkuva20a}, tuned separately for each setting.
Full implementation details, including hyperparameter tuning and architectures, are provided in \cref{sec:appen-simulation-tuning}.

\begin{table}[t]
    \centering
    \small
    \caption{\textbf{Comparison of map quality: distribution matching and transport cost.} Best values of \texttt{Wass} are \textbf{bolded} and `$\ast$' marks the failed cases due to algorithm instability. Results are averaged over five random trials, and standard deviations are reported in \cref{tab:simuls_full} of \cref{sec:appen-simulation-full}.}
    \label{tab:distmatch}
    \vskip 0.1in
    \begin{tabular}{lcccccc}
        \toprule
         &  \multicolumn{2}{c}{Gaussian $\to$ Rectangle}  &  \multicolumn{2}{c}{Gaussian $\to$ Ellipse}  &  \multicolumn{2}{c}{\textsc{AFHQv2} wild $\to$ cat}  \\
        \cmidrule(lr){2-3}\cmidrule(lr){4-5}\cmidrule(lr){6-7}
        Method  &  \texttt{Wass}  &  \texttt{TC}  &  \texttt{Wass}  &  \texttt{TC}  &  \texttt{Wass}  &  \texttt{TC}  \\
        \midrule
        \midrule
        1NN \cite{manole2024pluginestimationsmoothoptimal}  &  0.126  &  1095.67  &  0.207  &  1122.30  &  54.819  &  3296.54  \\
        LSOT \cite{seguy2018large}                          &  0.111  &  1095.69  &  0.172  &  1124.35  &  53.263  &  2627.29  \\
        ICNN \cite{pmlr-v119-makkuva20a}  &  {0.211}  & {1097.51}  &  0.171  &  1122.54  &  $\ast$  &  $\ast$  \\
        OTP \cite{choi2025overcoming}                        &  0.258  &  1099.17  &  0.596  &  1129.88  &  50.651  &  1408.23  \\
        DIOTM \cite{choi2025improving}                      &  0.123  &  1098.21  &  0.228  &  1123.99  &  51.000  &  1453.99  \\
MM-v2 \cite{korotin2021neural} & 0.120 & 1096.57 & 0.174 & 1125.20 & 205.078 & 40213.59 \\
W2 \cite{korotin2021neural} & 0.191 & 1096.44 & 0.163 & 1124.44 & 51.811 & 2366.89 \\
ENT \cite{pooladian2021entropic} & 0.209 & 1097.41 & 0.257 & 1125.98 & 45.708 & 1955.80 \\

        \rowcolor{RowHighlight}
        \textbf{BROT}  &  \textbf{0.110}  &  1095.65  &  \textbf{0.155}  &  1125.08  &  \textbf{45.484}  &  1993.85  \\
        \bottomrule
    \end{tabular}
    \vskip -0.1in
\end{table}

\paragraph{Results.}

\cref{tab:distmatch} compares the methods in terms of map quality on the three datasets, where we focus on \texttt{Wass} as the primary metric because a low \texttt{TC} alone does not imply a valid transport map, e.g., the trivial identity map attains the smallest \texttt{TC} but has $\mathbf{T}_{\#}P \neq Q$.
BROT attains the lowest $\texttt{Wass}$ on every dataset among all methods, while maintaining a comparable $\texttt{TC},$ suggesting that BROT provides a high-quality transport map.
Experimental details and full results are provided in \cref{sec:appen-simulation-full}.

In particular, BROT improves over 1NN in $\texttt{Wass}$ while showing similar $\texttt{TC},$ showing the practical benefit of replacing the discontinuous nearest neighbor operation with DNN.
ICNN also performs reasonably well on the two synthetic datasets, but diverges on \textsc{AFHQv2} ($d=4{,}096$).
This is due to optimization instability in the ICNN min-max objective on high-dimensional data, despite extensive hyperparameter tuning (see \cref{sec:appen-simulation-tuning}).
LSOT, OTP, and DIOTM show competitive $\texttt{Wass}$ on synthetic datasets but degrade on \textsc{AFHQv2}, indicating that BROT is a practically favorable DNN-based method.
MM-v2 and W2 also perform reasonably well on the synthetic datasets but show larger $\texttt{Wass}$ than BROT on \textsc{AFHQv2}.
ENT shows larger \texttt{Wass} on synthetic datasets than BROT.

The runtime comparison in \cref{tab:simuls_2} of \cref{sec:appen-simulation-full} shows that BROT trains faster than the baseline DNN-based algorithms with min-max learning algorithms (ICNN, OTP, DIOTM) while keeping inference cost lower than 1NN.
BROT also yields a smaller $\texttt{Wass}$ than the regularized counterpart (i.e., LSOT) in \cref{tab:distmatch}, and in fact can train faster than LSOT.
We further assess the transported representations in image space using decoded \textsc{AFHQv2} images and FID.
BROT also achieves the lowest FID, and the decoded image samples and full results are provided in \cref{sec:appen-decoded}.

\subsection{Application 1: Single-Cell Perturbation Prediction}\label{sec:app_cell}

\paragraph{Task setup.}

Predicting how an unperturbed cell would respond to a drug perturbation is a canonical task in modern computational biology \cite{bunne2023learning,chen_rivaud_park_tsou_charles_haliburton_pichiorri_thomson_2020,lotfollahi2019scgen}.
Since cell measurements are destructive, we only observe two unpaired data distributions, a control source distribution $\rho_c$ and a perturbed target distribution $\rho_k$, but never observe the same cell in both states.
Following \cite{bunne2023learning}, we formulate this task as estimating the OT map $\mathbf T^{(k)}:\rho_c \to \rho_k$ for each perturbation $k$.
The goal is to match the transported control distribution $(\mathbf T^{(k)})_{\#}\rho_c^{\textup{test}}$ to the treated cell distribution, where $\rho_c^{\textup{test}}$ denotes the test distribution of control cells unseen during training.
We compare BROT with several baselines, including CellOT \cite{bunne2023learning}, a state-of-the-art ICNN-based method for single-cell perturbation prediction.
We use \textsc{4i} benchmark dataset \cite{gut2018multiplexed} with $35$ drug perturbations, following the same setup as CellOT \cite{bunne2023learning}.
For each perturbation, we split the control and treated data into training and test data with an $80{:}20$ ratio.
Within the training split, we further hold out $20\%$ as validation data for hyperparameter selection, train maps on the remaining $80\%,$ and then evaluate them on the test data.

\paragraph{Implementation details.}

Following the evaluation protocol of \cite{bunne2023learning}, we compare the predicted (transported) control cell distribution with the ground-truth perturbed cell distribution.
Each cell is represented by a 48-dimensional protein-marker feature vector.
We use four performance metrics.
First, $\texttt{MMD}$ is the Gaussian-kernel maximum mean discrepancy (MMD) averaged over $50$ bandwidths, measuring the distributional discrepancy between the transported control and treated distributions.
Second, $L_2$ difference is the Euclidean distance between the feature means of the two distributions, measuring their feature-wise closeness.
Third, $r$-std is the Pearson correlation between the marker standard deviations of the transported control and treated distributions, measuring whether the transported control captures the marker-wise spread of the treated cells.
We additionally calculate the transport cost $\texttt{TC}$ between each unperturbed cell and its prediction.

\begin{table}[t]
    \centering
    \small
    \caption{\textbf{Comparison of single-cell perturbation prediction performance.} The results are averaged across the $35$ drug perturbations from \textsc{4i} dataset \cite{gut2018multiplexed}. Among the methods (excluding Identity) that satisfy $\texttt{MMD}\le 0.028$, the best value is \textbf{bolded} and the second-best is \underline{underlined} for each metric. Standard deviations are reported in \cref{tab:appen-single-cell-full} of \cref{sec:appen-cell-full}.}
    \label{tab:single_cell_4i__47_d__35_drugs_}
    \vskip 0.1in
    \begin{tabular}{lcccc}
        \toprule
        Method  &  \texttt{MMD} $\downarrow$  &  {$L_2$ difference $\downarrow$}  &  {$r$-std $\uparrow$}  &  {\texttt{TC} $\downarrow$}  \\
        \midrule\midrule
        Identity   &   0.028   &   {1.213}   &   {0.917}   &  0.000 \\
        1NN \citep{manole2024pluginestimationsmoothoptimal}   &   \textbf{0.002}   &   {\textbf{0.189}}   &   {0.990}   &  6.745 \\
        LSOT \citep{seguy2018large}   &   0.012   &   {0.287}   &   {0.982}   &  4.381 \\
        ICNN \citep{pmlr-v119-makkuva20a}   &  {\underline{0.006}}  &   {{0.258}}   &   {{0.987}}   &  4.329 \\
        OTP \citep{choi2025overcoming}   &   0.035   &   {1.435}   &   {0.917}   &  0.514 \\
        DIOTM \citep{choi2025improving}   &   \textbf{0.002}   &   {0.255}   &   {0.990}   &  3.993 \\
        MM-v2 \cite{korotin2021neural} & \textbf{0.002} & 0.226 & 0.990 & 3.984 \\
        W2 \cite{korotin2021neural} & \textbf{0.002} & 0.199 & 0.990 & \underline{3.934} \\
        ENT \cite{pooladian2021entropic} & \textbf{0.002} & 0.204 & \underline{0.992} & 6.291 \\

        \rowcolor{RowHighlight}
        \textbf{BROT}   &   {{\textbf{0.002}}}   &  {{\underline{0.194}}}  &   {{\textbf{0.993}}}   &  \textbf{3.928} \\
        \bottomrule
    \end{tabular}
    \vskip -0.1in
\end{table}

\paragraph{Results.}

\cref{tab:single_cell_4i__47_d__35_drugs_} reports the average results over the 35 drug perturbations.
It shows that BROT achieves the best $\texttt{MMD}$, $r$-std, and $\texttt{TC}$, while obtaining the second-best $L_2$ among the methods that satisfactorily match the target distribution ($\texttt{MMD}\le 0.028$).
In contrast, OTP fails to reach $\texttt{MMD}\le0.028$, which is the value of Identity (no transport).
Although DIOTM also attains $\texttt{MMD}\le 0.028$, it is noticeably weaker in terms of the $L_2$ difference.
Compared with 1NN, BROT achieves substantially lower \texttt{TC} with the same \texttt{MMD} level.
MM-v2, W2, and ENT attain similar $\texttt{MMD}$s as BROT, but show lower $r$-stds and higher \texttt{TC}s.

\subsection{Application 2: Unsupervised Domain Adaptation}\label{sec:app_da}

\paragraph{Task setup.}

In unsupervised domain adaptation, a classifier trained using labeled source data and unlabeled target data is evaluated on the target domain, whose distribution differs from that of the source.
Existing conventional OT-based approaches \cite{seguy2018large,7586038} first transport the source data toward the unlabeled target data and then train the classifier on the transported source data.
Although a more accurate OT map does not directly guarantee higher target-domain test accuracy, OT-based domain adaptation is a standard downstream evaluation \cite{seguy2018large,7586038}.
Following \cite{seguy2018large}, we run all methods on a fixed representation space using a pretrained visual encoder, where semantically similar images lie close and transport is more likely to preserve class information than in raw pixel space.

\paragraph{Implementation details.}

We evaluate on two benchmark datasets, \textsc{USPS} \cite{291440} $\to$ \textsc{MNIST} \cite{lecun2010mnist} and \textsc{VisDA-17} synthetic $\to$ real \cite{peng2017visdavisualdomainadaptation}, and with two pretrained encoders, CLIP-ViT/B-32 \cite{pmlr-v139-radford21a} and DINOv2-ViT/B-14 \cite{oquab2024dinov}.
After training the OT map, we train a classifier on the transported source data.
We report the per-class mean accuracy \citep{liang2020we} averaged over three random trials.
Further details such as classifier network selection and validation settings are given in \cref{sec:appen-da-full}.

\begin{table*}[t]
    \vskip -0.1in
    \centering
    \small
    \caption{\textbf{Comparison of unsupervised domain adaptation performance.} The performance is evaluated by the per-class mean accuracy (\%) on the target-domain test data. Best results are \textbf{bolded} and the second places are \underline{underlined}, and the standard deviations are in \cref{tab:appen-da-extended-std} of \cref{sec:appen-da-full}.}
    \label{tab:da_extended}
    \vskip 0.1in
    \setlength{\tabcolsep}{3.5pt}
\begin{tabular}{lcccc}
\toprule
  &  \multicolumn{2}{c}{\textsc{USPS} $\to$ \textsc{MNIST}}  &  \multicolumn{2}{c}{\textsc{VisDA-17} (synthetic $\to$ real)}  \\
\cmidrule(lr){2-3}\cmidrule(lr){4-5}
Method  &  CLIP-ViT/B-32  &  DINOv2-ViT/B-14  &  CLIP-ViT/B-32  &  DINOv2-ViT/B-14  \\
\midrule
    Identity  &  52.05  &  58.07  &  64.04  &  60.66  \\
    1NN \cite{manole2024pluginestimationsmoothoptimal}  &  80.34  &  75.25  &  74.17  &  74.39  \\
    LSOT \cite{seguy2018large}  &  30.12  &  52.20  &  76.87  &  \underline{81.57}  \\
    ICNN \cite{pmlr-v119-makkuva20a}  &  75.76  &  54.70  &  74.12  &  77.77  \\
    OTP \cite{choi2025overcoming}  &  49.98  &  57.36  &  76.15  &  64.51  \\
    DIOTM \cite{choi2025improving}  &  82.71  &  77.46  &  76.45  &  27.05  \\
MM-v2 \cite{korotin2021neural} & \underline{84.00} & \textbf{77.73} & \underline{76.95} & 74.81 \\
W2 \cite{korotin2021neural} & 80.50 & 75.49 & 74.70 & 70.40 \\
ENT \cite{pooladian2021entropic} & 79.76 & 76.21 & 73.88 & 74.28 \\

    \rowcolor{RowHighlight}
    \textbf{BROT}  &  \textbf{84.83}  & \underline{77.60} &  \textbf{79.34}  &  \textbf{81.59}  \\
\bottomrule
\end{tabular}
    \vskip -0.1in
\end{table*}

\paragraph{Results.}

\cref{tab:da_extended} reports the per-class mean accuracy on the target domain.
BROT performs best on VisDA with both encoders and on USPS $\to$ MNIST with CLIP, while MM-v2 is the strongest baseline and slightly outperforms BROT by $0.13$pp on USPS $\to$ MNIST with the DINOv2 encoder.
These results suggest that the accurate estimation of the OT maps using BROT translates into improved domain adaptation performance under various foundation visual encoders and datasets.
Furthermore, we compare BROT with several end-to-end OT-based domain adaptation methods that jointly train the encoder, classifier, and transport plan.
The results are reported in \cref{tab:appen-da-end2end} of \cref{sec:appen-da-full}, further showing that BROT with fixed representations is competitive with these end-to-end methods.

\section{Conclusion and Discussion}\label{sec:discussion}

We proposed BROT, a two-step algorithm for estimating the OT map that first computes the unregularized OT plan between the empirical source and target distributions and then fits a Lipschitz-constrained DNN to the induced barycentric targets by least-squares regression.
On the theoretical side, we show that the DNN estimator of BROT attains the minimax convergence rate when the ground-truth OT map is Lipschitz.
On the empirical side, BROT (i) attains the smallest estimation error, (ii) produces high-quality maps, and (iii) shows improved performance on two downstream applications.

We considered a Lipschitz-constrained DNN based on the empirical OT plan using full data, whereas our practical implementation uses an MLP with a Jacobian penalty and, for large-scale experiments, computes the empirical OT plan using chunked data.
Analyzing the MLP with a Jacobian penalty and the statistical effect of the chunking are promising future directions.

Further, extending the analysis to smoother classes, such as H\"older classes, would provide a more general statistical guarantee.
Several existing estimators \cite{hutter2020minimaxestimationsmoothoptimal,10.1214/24-AOS2482} are already known to be minimax optimal for estimating H\"older-smooth OT maps, but they are not accompanied by practical training algorithms.
Designing a practically trainable DNN estimator that is minimax optimal for such smoother classes is therefore a promising direction for future work.
The statistical optimality of existing neural estimators, including ICNN-based and semi-dual estimators \cite{tarasov2026statistical}, also remains open.

\clearpage

\section*{Acknowledgement}

KK was partly supported by Institute for Information \& communications Technology Planning \& Evaluation (IITP) grant funded by the Korea government (MSIT) (RS-2019-II190075, Artificial Intelligence Graduate School Support Program (KAIST)),
Institute of Information \& communications Technology Planning \& Evaluation (IITP) grant funded by the Korea government (MSIT) (No.2022-0-00713, Meta-learning applicable to real-world problems),
and
the InnoCORE program of the Ministry of Science and ICT (No. N10250156).

YK was partly supported by Institute of Information \& communications Technology Planning \& Evaluation (IITP) grant funded by the Korea government (MSIT) [NO.RS-2021-II211343, Artificial Intelligence Graduate School Program (Seoul National University)],
the National Research Foundation of Korea (NRF) grant funded by the Korea government (MSIT) (RS-2025-00556079),
and
the Institute of Information \& Communications Technology Planning \& Evaluation (IITP) with a grant funded by the Ministry of Science and ICT (MSIT) of the Republic of Korea in connection with the Development of Unified Reasoning Technology Mimicking Human Cognition for Hierarchical Understanding and Unbounded Problem Solving (No. RS-2026-25522672).

The authors declare no competing interests.
The authors truly thank the anonymous reviewers and the area chair for their thoughtful comments, which helped improve this work.
KK sincerely thanks Woncheol Jang, Johan Lim, Minwoo Chae, and Gunwoong Park for their insightful discussions and valuable feedback during the development of this work.

\renewcommand{\bibfont}{\small}
\bibliography{references}
\bibliographystyle{unsrt}

\clearpage
\appendix
\crefalias{section}{appsec}
\crefalias{subsection}{appsubsec}

\section{Theoretical Studies}\label{sec:appen-maths}

This section introduces precise mathematical definitions, assumptions, and theoretical background on OT map to support \cref{sec:preliminary}, and provides proofs for our main theoretical result in \cref{sec:minimax}.

\subsection{Notations, Definitions, and Assumptions}\label{sec:appen-notation}

\paragraph{Notations.}\label{notation:norms}

For $\nu\in[1,\infty)$ and $v=(v_1,\ldots,v_d)^\top\in\mathbb R^d$, define $\|v\|_\nu=(\sum_{j=1}^d|v_j|^\nu)^{1/\nu}$.
We write $\|v\|=\|v\|_2$ and $\|v\|_\infty=\max_j|v_j|$; the same conventions apply to vectors of other finite dimensions.
For a matrix $A$, $\|A\|_{\textup{op}}$ is its Euclidean operator norm and $\|A\|_F=(\sum_{i,j}A_{ij}^2)^{1/2}$ is its Frobenius norm.
For a scalar function $f$ on $\Omega$, a subset $D\subseteq\Omega$, and a probability measure $P$, define
$$
\begin{aligned}
\|f\|_{L^\nu(\Omega)}&=\left(\int_\Omega|f(x)|^\nu\,dx\right)^{1/\nu},\\
\|f\|_{L^\nu(P)}&=\left(\int_\Omega|f(x)|^\nu\,dP(x)\right)^{1/\nu},\quad
\|f\|_{\infty,D}=\sup_{x\in D}|f(x)|.
\end{aligned}
$$
The spaces $L^\nu(\Omega)$ and $L^\nu(P)$ consist of measurable functions with finite respective norms, and $\|f\|_\infty$ abbreviates $\|f\|_{\infty,\Omega}$.
For vector-valued $f=(f_1,\ldots,f_{d_{\textup{out}}})^\top$, the $L^\nu$ norms use $\|f(x)\|_2$ inside the integral, while $\|f\|_{\infty,D}:=\max_j\sup_{x\in D}|f_j(x)|$.
The notation $\|f\|_{L^\infty(D)}$ also denotes this uniform norm for the continuous functions used below.

For $k\in\mathbb N$, write $[k]=\{1,\ldots,k\}$.
For any $a,b\in\mathbb{R}$, denote $a\vee b=\max\{a,b\}$, $a\wedge b=\min\{a,b\}$, and $a_+=a\vee 0$.
Let $\lfloor a\rfloor$ denote the greatest integer less than or equal to $a$.
Given a set $\Omega \subset \mathbb{R}^{d},$ the diameter of $\Omega$ is defined as $\textup{diam}(\Omega)=\sup\{\|x-y\|: x,y\in\Omega\}$.
For all $x\in\mathbb{R}^d$ and $\varepsilon>0$, let $\mathbb B(x,\varepsilon)=\{ y\in\mathbb{R}^d:\|x-y\|\le\varepsilon \}$ be the $\epsilon$-ball with center $x.$
For two given sequences $(a_n)_{n\ge1}$ and $(b_n)_{n\ge1}$ depending on $n,$ we write $a_n\lesssim b_n$ if there exists some $C>0$ not depending on $n$ such that $a_n\le Cb_n$ for all $n\ge1.$
We write $a_n\asymp b_n$ if $b_n\lesssim a_n\lesssim b_n$.

\paragraph{H\"older spaces.}

Let $\Omega \subset \mathbb{R}^d$ be a support.
Uniform norms on $\Omega$ and its subsets use the conventions above.
We write $\mathbb{N}_0 := \mathbb{N} \cup \{0\}$.
Throughout this study, $\|\cdot\|$ denotes the Euclidean norm on $\mathbb{R}^d$.

\begin{definition}[Partial derivatives]

Let $\bm{\alpha} = [\alpha_1,\ldots,\alpha_d]^\top \in \mathbb{N}_0^d$ with $|\bm{\alpha}| := \sum_{i=1}^d \alpha_i$.
For $x = [x_1,\ldots,x_d]^\top \in \Omega$, we define the partial derivative of $f$ of order $\bm{\alpha}$ as
\begin{equation}
    \partial^{\bm{\alpha}} f(x)
    := \frac{\partial^{|\bm{\alpha}|} f(x)}{\partial x_1^{\alpha_1} \cdots \partial x_d^{\alpha_d}}.
\end{equation}

\end{definition}

For $k \in \mathbb{N}_0$, let $\mathcal{C}^k(\Omega)$ denote the space of $k$-times continuously differentiable functions such that the partial derivatives $\partial^{\bm{\alpha}} f$ exist and are continuous for all $|\bm{\alpha}| \le k$.

For a scalar $f$ and derivative order $\ell$, write $\|\nabla^\ell f\|_{\infty,D}:=\max_{|\bm\alpha|=\ell}\|\partial^{\bm\alpha}f\|_{\infty,D}$; vector-valued derivatives also take the maximum over output coordinates.

\begin{definition}[H\"older spaces $\mathcal{C}^{\beta}(\Omega)$ with order $\beta>0$]

    Let $\beta > 0$ be noninteger.
    Define $s := \lfloor \beta \rfloor \in \mathbb{N}_0$ and $r := \beta - \lfloor \beta \rfloor \in (0,1)$.
    For $f:\Omega \to \mathbb{R}$, define
    \begin{equation}
        \|f\|_{\mathcal{C}^{s}(\Omega)} := \max_{|\bm{\alpha}|\le s}\ \|\partial^{\bm{\alpha}} f\|_{\infty}
        \textup{ and }
        [f]_{\mathcal{C}^{s,r}(\Omega)} := \max_{|\bm{\alpha}|=s}\ \sup_{\substack{x,y\in\Omega\\ x\neq y}}
        \frac{|\partial^{\bm{\alpha}} f(x) - \partial^{\bm{\alpha}} f( y)|}{\|x- y\|^{r}}.
    \end{equation}
    and define the H\"older norm of order $\beta$ as
    $ \|f\|_{\mathcal{C}^{\beta}(\Omega)} := \|f\|_{\mathcal{C}^{s}(\Omega)} + [f]_{\mathcal{C}^{s,r}(\Omega)}. $
    The H\"older space of order $\beta$ is then defined as
    \begin{equation}
        \mathcal{C}^{\beta}(\Omega)
        :=
        \bigl\{  f \in \mathcal{C}^{s}(\Omega): \|f\|_{\mathcal{C}^{\beta}(\Omega)} < \infty \bigr\}.
    \end{equation}
    Furthermore, for $H > 0$, we define the (closed) H\"older ball of radius $H$ as
    \begin{equation}
        \mathcal{C}^{\beta}(\Omega, H)
        :=
        \bigl\{ f \in \mathcal{C}^{\beta}(\Omega): \|f\|_{\mathcal{C}^{\beta}(\Omega)} \le H \bigr\}.
    \end{equation}
     For an integer $\beta=k$, we use the space $\mathcal C^k(\Omega)$ and its norm defined above.

\end{definition}

\begin{definition}[H\"older spaces of vector-valued maps]

    For a vector-valued map $\mathbf{T} = (T_1,\ldots,T_{d_{\textup{out}}})^\top : \Omega \to \mathbb{R}^{d_{\textup{out}}}$ with $d_{\textup{out}} \in \mathbb{N}$, set
    \begin{equation}
        \|\mathbf{T}\|_{\mathcal{C}^{\beta}(\Omega)} := \max_{1\le \ell\le d_{\textup{out}}}\ \|T_\ell\|_{\mathcal{C}^{\beta}(\Omega)}.
    \end{equation}
    We then write
    \begin{equation}
        \mathbf{T} \in \mathcal{C}^{\beta}(\Omega)
        \ \Longleftrightarrow\
        T_\ell \in \mathcal{C}^{\beta}(\Omega)\ \text{ for all }\ \ell \in [d_{\textup{out}}].
    \end{equation}

\end{definition}

\paragraph{Regularity conditions.}

In addition, the standard regularity conditions on the support $\Omega$ are:

\begin{itemize}
    \item[(S1)]
    $\Omega$ is a compact, convex set with nonempty interior such that $\Omega \subseteq [0, 1]^{d}$.

    \item[(S2)]
    $\Omega$ is a standard set in the sense that there exist $\epsilon_{0}, \delta_{0} > 0$ such that for all $x \in \Omega$ and $\epsilon \in (0, \epsilon_{0})$,
    $
    \operatorname{Leb}\bigl(\mathbb B(x, \epsilon) \cap \Omega\bigr)
    \ge \delta_{0}  \operatorname{Leb}\bigl(\mathbb B(x, \epsilon)\bigr),
    $
    where $\operatorname{Leb}$ is the Lebesgue measure on $\mathbb{R}^{d}$ and $\mathbb B(x,\epsilon)$ is the Euclidean ball of radius $\epsilon$ centered at $x$.
\end{itemize}

\paragraph{Remark on the regularity class of \cref{assumption-3}.}

The two-sided Hessian bound $\lambda^{-1} I_d \preceq \nabla^2 \phi_0(x) \preceq \lambda I_d$ in \cref{assumption-3} makes $\phi_0$ both $\lambda$-smooth and $\lambda^{-1}$-strongly convex on $\Omega$, so $\mathbf{T}_0 = \nabla \phi_0$ is in fact $\lambda$-bi-Lipschitz: both $\mathbf{T}_0$ and $\mathbf{T}_0^{-1}$ are $\lambda$-Lipschitz.
These assumptions follow the standard regularity regime used in minimax OT-map estimation \cite{hutter2020minimaxestimationsmoothoptimal,manole2024pluginestimationsmoothoptimal,10.1214/24-AOS2482}, and we refer to it as `Lipschitz OT map' throughout this paper, following the convention.

\subsection{Existence, Uniqueness, and Duality of Optimal Transport Maps}\label{sec:appen-dual_detail}

Based on the Kantorovich relaxation in \cref{eq:kantorovich}, the $2$-Wasserstein distance admits a convenient dual formulation.

\begin{theorem}[Kantorovich duality for quadratic cost]\label{thm:kantorovich_duality}

    Assume that $P$ and $Q$ have finite second moments.
    Then
    \begin{equation}\label{eq:dual_full}
        \begin{split}
            W_2^2(P,Q)
            & = \inf_{\Gamma\in\Pi(P,Q)} \int \|x-y\|^2 d\Gamma(x,y)
            \\
            & = \sup_{\substack{\varphi\in L^1(P), \psi\in L^1(Q) \\ \varphi(x)+\psi(y)\le \|x-y\|^2\ \forall (x,y)\in\Omega\times\Omega}}
            \left\{\int \varphi dP+\int \psi dQ \right\}
            \\
            & = \sup_{u\in L^1(P)} \left\{\int u dP+\int u^{c} dQ \right\},
        \end{split}
    \end{equation}
    where the $c$-transform is $u^c(y):=\inf_{x\in\Omega}\{\|x-y\|^2-u(x)\}$.
    Moreover, writing $\phi$ for a convex potential and $\phi^*$ for its Legendre-Fenchel conjugate, we obtain
    \begin{equation}\label{eq:dual_conjugate}
        \begin{split}
            W_2^2(P,Q)
            = \int \|x\|^2 dP(x)+\int \|y\|^2 dQ(y)
            - 2\inf_{\phi: \textup{convex}}
            \left\{\int \phi dP+\int \phi^* dQ \right\}.
        \end{split}
    \end{equation}

\end{theorem}

Thus, computing $W_2^2(P,Q)$ reduces to solving
$$ \inf_{\phi: \textup{convex}} \Bigl\{\int \phi dP + \int \phi^* dQ \Bigr\}. $$
More detailed derivations and proofs are given in previous literature, e.g., \cite{peyre2020computationaloptimaltransport,villani2003topics,villani2008optimal}.

Along with this duality result, we recall the fundamental existence and uniqueness theorem for the OT map.
Let $\mathcal{P}_{\textup{ac}}(\Omega)$ denote the set of absolutely continuous probability distributions on $\Omega$.

\begin{theorem}[Brenier's Theorem]\label{thm:brenier}

    Suppose that $P\in\mathcal{P}_{\textup{ac}}(\Omega)$ and that $P,Q$ have finite second moments.
    Then there exists a convex function $\phi_0$ such that the unique ($P$-almost surely) $W_2$-optimal transport map is given by
    $ \mathbf{T}_0 = \nabla\phi_0. $
    Moreover, an optimal dual pair for \cref{eq:dual_full} is $ \varphi_{0}(x)=\|x\|^2-2\phi_0(x), \psi_{0}(y)=\|y\|^2-2\phi_0^*(y), $
    and $(\phi_0,\phi_0^\ast)$ is the pair (potential function and its conjugate) generating the map $\mathbf{T}_0=\nabla\phi_0$.

\end{theorem}

\begin{corollary}[Semi-dual maximizer]\label{cor:semi_dual_potential}
    Assume that $P\in\mathcal{P}_{\textup{ac}}(\Omega)$ and $Q$ have finite second moments.
    Let $u^\star$ be a $c$-concave maximizer of the semi-dual problem in \cref{eq:dual_full}, and define $ \phi_0(x) := \tfrac{1}{2}\big(\|x\|^2-u^\star(x)\big), x \in \Omega. $
    Then $\phi_0$ is convex, $(\phi_0,\phi_0^\ast)$ attains \cref{eq:dual_conjugate}, and the unique optimal transport map satisfies
    $$
    \mathbf{T}_0(x) = \nabla\phi_0(x), \quad P\textup{-a.s.}
    $$
    Conversely, if $\phi_0$ is any convex function such that $\mathbf{T}_0=\nabla\phi_0$ is $W_2$-optimal, then
    $$ u^\star(x) = \|x\|^2-2\phi_0(x) $$
    is a maximizer of the semi-dual problem in \cref{eq:dual_full}.
\end{corollary}

The dual problem selects the highest pair of ``roofs'' $(\varphi,\psi)$ that lie below the cost surface $c(x,y)=\|x-y\|^2$ everywhere.
At optimality, mass is transported only along the contact set where the roofs touch the cost.
For the quadratic cost, these roofs are generated by a single convex potential $\phi$ via conjugacy, and the optimal transport is given by the gradient field $x\mapsto\nabla\phi(x)$.

\paragraph{Examples of Lipschitz OT maps.}

Besides the example in \cref{sec:rate}, affine maps provide a simple class of Lipschitz OT maps.
Let $A\in\mathbb R^{d\times d}$ be symmetric positive definite and $c\in\mathbb R^d$.
Then $\mathbf T_0(x)=Ax+c$ is the gradient of
$\phi_0(x)=x^\top Ax/2+c^\top x$ and is therefore the OT map from $P$ to $\mathbf T_0\# P$.
It satisfies \cref{assumption-3} with
$\lambda=\lambda_{\max}(A)\vee\lambda_{\min}(A)^{-1}$.

The OT map between two Gaussian distributions is also affine with a symmetric positive definite linear coefficient matrix, and is therefore Lipschitz.

\subsection{Deep Neural Network Function}\label{sec:appen-dnn}

\begin{definition}[DNN function]\label{def:mlp_ftn}

    A deep neural network (DNN) with $(L, \mathbf{d})$ architecture consists of
    \begin{itemize}[topsep=0pt, leftmargin=1.5em, labelsep=0.5em]
        \item $L \in \mathbb{N}$: the number of hidden layers;
        \item $\mathbf{d} = [ d_{0}, \ldots, d_{L+1} ]^{\top} \in \mathbb{N}^{L+2}$: the layer widths, where $d_{0} = d$ and $d_{L+1}$ are the input and output dimensions, respectively;
    \end{itemize}
    is mathematically defined as
    \begin{equation}\label{eq:mlp}
        f^{\textup{DNN}}_{\theta, \sigma}(x)
        := A_{L+1} \circ \sigma \circ A_{L} \circ \cdots \circ \sigma \circ A_{1}(x),
    \end{equation}
    where
    \begin{itemize}[topsep=0pt, leftmargin=1.5em, labelsep=0.5em]
        \item $A_{l}(x) := \mathbf{W}_{l} x + b_{l}$ with $\mathbf{W}_{l} \in \mathbb{R}^{d_{l} \times d_{l-1}}$ and $b_{l} \in \mathbb{R}^{d_{l}}$ for $l \in [L+1]$;
        \item $\sigma$ is an element-wise activation function (e.g., ReLU: $\sigma (t) = \max \{t, 0\}$ or ReQU: $\sigma (t) = \max \{t, 0\}^{2}$);
        \item $\theta := [ \textup{vec}(\mathbf{W}_{1})^{\top}, \ldots, \textup{vec}(\mathbf{W}_{L+1})^{\top}, b_{1}^{\top}, \ldots, b_{L+1}^{\top} ]^{\top}$ collects all parameters.
    \end{itemize}

\end{definition}

Given $f^{\textup{DNN}}_{\theta, \sigma}$ with $(L,\mathbf{d})$ in \cref{eq:mlp}, let $d_{\max} := \max \{ d_{1}, \ldots, d_{L} \}$ be the width of $f^{\textup{DNN}}_{\theta, \sigma}$.
\begin{definition}[Class of DNN functions]\label{def:mlp_class}

    Given an activation function $\sigma$, the class of DNN functions with at most $L$ layers (depth), at most $W$ neurons per layer (width), and output dimension $d_{\textup{out}} \in \mathbb{N}$ is denoted as
    \begin{equation}
        \begin{split}
            &\Phi^{\textup{DNN}}(L, W; \sigma, d_{\textup{out}})
            :=
            \big\{
            f^{\textup{DNN}}_{\theta, \sigma} \text{ with $(L^{\prime}, \mathbf{d})$ as in \cref{eq:mlp}}:\\
            &\quad L^{\prime}\le L,\quad d_{L^{\prime}+1}=d_{\textup{out}},\quad d_{\max} \le W,
            \Vert \theta \Vert_{\infty} \le 1
            \big\}.
        \end{split}
    \end{equation}
    We abbreviate $\Phi^{\textup{DNN}}(L, W; \sigma, d_{\textup{out}})$ by $\Phi^{\textup{DNN}}(L, W; \sigma)$ when $d_{\textup{out}}$ is clear from context.

\end{definition}

As we assume the true OT map is Lipschitz continuous, we also enforce the Lipschitz continuity to the DNN functions.
That is, for a given Lipschitz constant $\lambda > 0$, we define
$$
\Phi^{\textup{DNN}}_{\lambda}
:= \big\{ f \in \Phi^{\textup{DNN}}(L, W; \sigma) : \operatorname{Lip}(f) \le \lambda \big\},
$$
Here $\operatorname{Lip}(f)$ is the difference-quotient Lipschitz constant on $\Omega$ defined in \cref{sec:intro_cdot}.

\subsection{Proofs}\label{sec:appen-proofs-brot}

In this section, we present the full proof of our main theoretical result (\cref{thm:dnn_main}), with technical lemmas used to prove \cref{thm:dnn_main}.
Recall the assumptions used in the main body.

\begin{itemize}[topsep=0pt, leftmargin=1.5em, labelsep=0.5em]
    \item
    (S1)
    $\Omega$ is a compact, convex set with nonempty interior such that $\Omega \subseteq [0, 1]^{d}.$

    \item
    (S2)
    $\Omega$ is a standard set in the sense that
    there exist $\epsilon_{0}, \delta_{0} > 0$ s.t. for all $x \in \Omega$ and $\epsilon \in (0, \epsilon_{0}),$
    we have $\operatorname{Leb}(\mathbb B(x, \epsilon) \cap \Omega) \ge \delta_{0} \operatorname{Leb}(\mathbb B(x, \epsilon)),$
    where $\operatorname{Leb}$ is the Lebesgue measure on $\mathbb{R}^{d}.$

    \item
    \cref{assumption-1-2}: $P$ and $Q$ are absolutely continuous on $\Omega$ with densities $p$ and $q$ satisfying $\gamma^{-1} \le p(x), q(x) \le \gamma$ for all $x \in \Omega$ and some $\gamma > 0$.

    \item
    \cref{assumption-3}: the potential $\phi_{0}$ is convex, twice continuously differentiable on $\Omega$, and satisfies $\lambda^{-1} I_{d} \preceq \nabla^{2} \phi_{0}(x) \preceq \lambda I_{d}$ for all $x \in \Omega$, so that $\mathbf{T}_{0} = \nabla \phi_{0}$ is $\lambda$-Lipschitz.
\end{itemize}

\paragraph{Technical Lemmas.}\label{sec:appen-lemmas}

\begin{lemma}[Proposition 14 of \cite{manole2024pluginestimationsmoothoptimal}]
\label{lem:prop_14_manole}
    Let $\hat\Gamma=[\hat\gamma_{ij}]$ be an optimal transport plan between
    $P_n$ and $Q_m$ for the quadratic cost.
    Under \cref{assumption-1-2} and \cref{assumption-3}, with
    $\Omega=[0,1]^d$, we have
    \begin{equation}
        \mathbb{E}\left(
        \sum_{i=1}^{n}\sum_{j=1}^{m}
        \hat\gamma_{ij}
        \|\mathbf T_0(X_i)-Y_j\|^2
        \right)
        \lesssim
        \bar\kappa_{\widetilde n}
        :=
        \begin{cases}
            \widetilde n^{-1}, & d=1,\\
            \widetilde n^{-1}(\log\widetilde n)^2, & d=2,\\
            \widetilde n^{-2/d}, & d\ge3.
        \end{cases}
    \end{equation}
    where $\widetilde n=n\wedge m.$
\end{lemma}

\begin{proof}[Proof of \cref{lem:prop_14_manole}]
    Under \cref{assumption-3}, we have
    \begin{equation}
        \frac{1}{2\lambda} \Vert x - y \Vert^{2} \le \phi_{0}(y) - \phi_{0}(x) - \langle \mathbf{T}_{0}(x), y - x \rangle \le \frac{\lambda}{2} \Vert x - y \Vert^{2},
        \quad
        \forall x, y \in \Omega.
    \end{equation}
    Using these inequalities and the fact that
    \begin{equation}
        \mathbb{E} \left( \sum_{i=1}^{n} \sum_{j=1}^{m} \widehat{\gamma}_{ij} \Vert \mathbf{T}_{0}(X_{i}) - Y_{j} \Vert^{2} \right)
        \asymp
        \mathbb{E} \left( W_{2}^{2}(P_{n}, Q_{m}) - W_{2}^{2}(P, Q) \right)
    \end{equation}
    concludes the proof.
    See \cite{manole2024pluginestimationsmoothoptimal} for the detailed proof.
\end{proof}

\begin{lemma}[Lemma 42 of \cite{manole2024pluginestimationsmoothoptimal}]\label{lem:voronoi_mass_diam}
    Suppose that $\Omega$ satisfies conditions \textup{(S1)}-\textup{(S2)} and assume \cref{assumption-1-2}.
    Let $V_{i} = \{ x \in \Omega : \Vert x - X_i \Vert \le \Vert x - X_j \Vert, \forall j \neq i \}$ for $i \in [n].$
    Then, there exist constants $C_1$ and $C_2>0$, depending only on $d,\gamma,\epsilon_0,\delta_0$, such that

    \begin{enumerate}
        \item For all $\delta\in(0,1)$,
        \begin{equation}
            \mathbb P\left(
            \max_{1\le i\le n} P(V_i) \;\ge\; \frac{C_1}{n}\Big[d\log n+\log(1/\delta)\Big]
            \right) \;\le\; \delta.
        \end{equation}

        \item
        \begin{equation}
            \mathbb E\left[ \max_{1\le i\le n} \textup{diam}(V_i)^2 \right]
            \le C_2\left(\frac{\log n}{n}\right)^{2/d}.
        \end{equation}
\end{enumerate}
\end{lemma}

\begin{proof}[Proof of \cref{lem:voronoi_mass_diam}]
    This is a direct corollary of the following lemma, as shown in \cite{manole2024pluginestimationsmoothoptimal}:
    There exists a universal constant $C > 0$ such that, for every Euclidean ball $\mathbb B$,
    $P(\mathbb B) \ge \frac{C}{n} \left( d \log n + \log (1/\delta) \right)$ implies $P_{n}(\mathbb B) > 0$
    with probability $1 - \delta$ at least for any $\delta \in (0, 1).$
\end{proof}
\begin{lemma}[Lemma 44 of \cite{manole2024pluginestimationsmoothoptimal}]\label{lem:lambda_lipschitz_partition}
    Under \cref{assumption-1-2}, \cref{assumption-3}, (S1), and (S2), with $\Omega=[0,1]^d$, for any $\lambda$-Lipschitz map $f:\Omega\to\Omega$, we have
    \begin{equation}
        \mathbb E
        \left( \sum_{i=1}^n \int_{V_i} \|f(X_i)-f(x)\|^2 dP(x) \right)
        \lesssim \left( \frac{\log n}{n} \right)^{2/d},
    \end{equation}
    and
    \begin{equation}\label{eq:weighted-residual-rate}
        \mathbb E
        \left(
        n \sum_{i=1}^{n} \sum_{j=1}^{m} \widehat{\gamma}_{ij} \Vert \mathbf{T}_{0}(X_{i}) - Y_{j} \Vert^{2}
        \left( \max_{1\le i\le n} P(V_i) \right)
        \right)
        \lesssim
        (\log n)\,\bar\kappa_{\widetilde n},
    \end{equation}
    where $\bar\kappa_{\widetilde n}$ is defined in \cref{lem:prop_14_manole}.
\end{lemma}
The next four lemmas construct a comparator from a smooth potential, which is used as an alternative comparator in the proof of \cref{thm:dnn_main}.
\begin{lemma}\label{lem:requ_simul_correct}
    Fix $\beta>2$ and $d_{\textup{out}},d\in\mathbb N$.
    For any $H>0,$ an integer $K \ge 2,$ and any $f:[0,1]^d\to\mathbb R^{d_{\textup{out}}}$ with $f \in \mathcal{C}^\beta ([0,1]^d,H)$, there exists a DNN function $\mathbf{T}_f:[0,1]^d\to\mathbb R^{d_{\textup{out}}}$ in $\Phi^{\textup{DNN}}(L, W; \sigma)$ where
    $
    L \gtrsim \log d+\lfloor\beta\rfloor+\log\log H,
    W \gtrsim (d_{\textup{out}}\vee d) (K+\lfloor\beta\rfloor)^d
    $
    and $\sigma(t) = (t)_{+}^{2}$ (i.e., the ReQU activation),
    such that for all $\ell\in\{0,1,\dots,\lfloor\beta\rfloor\}$,
    \begin{equation}
        \sup_{x \in [0, 1]^{d}} \| \nabla^\ell f(x) - \nabla^\ell \mathbf{T}_f(x) \|
        = \| \nabla^\ell f - \nabla^\ell \mathbf{T}_f \|_{\infty, [0,1]^d}
        \le C(\beta, d, H) \frac{1}{K^{\beta - \ell}}
    \end{equation}
    where $C(\beta, d, H) > 0$ is a constant depending on $\beta, d,$ and $H.$
    In particular, for $d_{\textup{out}} = 1$,
    $$
    \|\nabla f - \nabla \mathbf{T}_f \|_{\infty, [0,1]^d} \lesssim K^{-(\beta-1)},
    \quad
    \|\nabla^2 f - \nabla^2 \mathbf{T}_f \|_{\infty, [0,1]^d} \lesssim K^{-(\beta-2)}.
    $$
\end{lemma}

\begin{proof}[Proof of \cref{lem:requ_simul_correct}]
    This is a direct consequence of Theorem 1 in \cite{BELOMESTNY2023242}.
\end{proof}
\begin{lemma}\label{lem:repu-partial-ours}
    Let $f\in \Phi^{\textup{DNN}} (L,W; \textup{ReQU}, 1)$ be any given ReQU DNN on $[0,1]^d$ such that there exist $M_f,M_{\nabla f}>0$ satisfying
    $$ \|f\|_{\infty, [0,1]^d}\le M_f,
    \textup{ and }
    \max_{1\le j\le d} \|\partial_{x_j}f\|_{\infty, [0,1]^d}\le M_{\nabla f}. $$
    Then, the gradient
    $ \nabla f=(\partial_{x_1}f,\dots,\partial_{x_d}f):[0,1]^d\to\mathbb R^d $
    is exactly implemented by a DNN with mixed activations (ReLU and ReQU), that is,
    \begin{equation}
        \nabla f \in \Phi^{\textup{DNN}} \left(3L+3, 6dW; \textup{Mix of ReLU and ReQU}, d \right).
    \end{equation}
\end{lemma}

\begin{proof}[Proof of \cref{lem:repu-partial-ours}]
    This is a consequence of Theorem 1 in \cite{shen2024differentiableneuralnetworksrepu}.
    We here provide a proof sketch, and more details can be found in \cite{shen2024differentiableneuralnetworksrepu}.

    The main idea is the following exact representations of polynomials, that are realizable by affine maps with ReLU or/and ReQU activations.
    Let $\sigma_{1}(t) = t_{+}$ and $\sigma_{2}(t) = (t_{+})^{2}$ be the ReLU and ReQU activation functions, respectively.
    Then, we have
    $ z=\sigma_1(z)-\sigma_1(-z)$ and $z^2=\sigma_2(z)+\sigma_2(-z). $
    Let the activation of layer $\ell$ be
    $z^\ell:=A_\ell(h^{\ell-1})=W_\ell h^{\ell-1}+b_\ell$ and $h^\ell:=\sigma_2(z^\ell)$, with $h^0:=x$.
    By the chain rule, for any $u\in\mathbb R$, we have
    $
    \nabla_x (u f(x))
    = \big(W_1^\top D^1(x) W_2^\top D^2(x)\cdots W_L^\top D^L(x) W_{L+1}^\top\big) u,
    $
    where $D^\ell(x):=\textup{diag}(\sigma_2'(z^\ell(x)))$.
    Since $ \sigma_2'(t)=2 \sigma_1(t), $
    we have $D^\ell(x)=2 \textup{diag}(\sigma_1(z^\ell(x)))$, i.e., a diagonal ReLU-gating matrix.
    Thus the differentiation circuit consists only of linear maps, ReLU, and element-wise products.
    Hence, a single backpropagation step is realized exactly by a combination of a linear map, ReLU/ReQU, and a linear map block.
    The construction of the remaining layers for the DNN with mixed activations follows \cite{shen2024differentiableneuralnetworksrepu}, which yields the parameters $(3L+3, 6dW).$
\end{proof}

Fix $\beta>2$ and $H>0$.
Let $\mathcal J_\lambda$ denote the class of twice continuously
differentiable convex potentials $\phi$ on $[0,1]^d$ such that
$ \lambda^{-1}I_d \preceq \nabla^2\phi(x) \preceq \lambda I_d $
and
$ \nabla\phi([0,1]^d)\subseteq[0,1]^d. $
Define the empirical OT-weighted least-squares comparator
$$
\hat\phi^{\textup{LS}}
=
\argmin_{\phi\in\mathcal J_\lambda\cap
\mathcal C^\beta([0,1]^d,H)}
\sum_{i=1}^n\sum_{j=1}^m
\hat\gamma_{ij}
\|Y_j-\nabla\phi(X_i)\|^2,
$$
and set
$$
\widehat{\mathbf T}^{\textup{LS}}
:=
\nabla\hat\phi^{\textup{LS}}.
$$
Then $\hat\phi^{\textup{LS}}\in \mathcal C^\beta([0,1]^d,H)$ holds by construction.
Note that for continuously differentiable maps on the cube, we can rewrite the Lipschitz constant of $f$ as $\operatorname{Lip}(f)=\sup_{x\in[0,1]^d}\|Df(x)\|_{\textup{op}}$.
\begin{lemma}\label{lem:lipaware_requ_mixedrepu_exact}
    Let $K\ge2$. Assume that there exists a ReQU DNN $h_{\textup{LS}} \in \Phi^{\textup{DNN}}(L, W; \textup{ReQU},1)$ such that, for $\ell=1,2$,
    \begin{equation}\label{eq:pre-approx}
        \|\nabla^\ell h_{\textup{LS}} - \nabla^\ell \hat\phi^{\textup{LS}}\|_{L^\infty([0,1]^d)} \le C_0 K^{-(\beta-\ell)},
    \end{equation}
    where $C_0>0$ depends only on $(\beta,H,d)$ and universal constants from \cref{lem:requ_simul_correct}.
    Define $\nabla h_{\textup{LS}} : [0,1]^d\to\mathbb R^d$. Then there exist constants $C_1,C_2,c>0$, depending only on $(\beta,H,d)$ and universal constants, such that
    \begin{equation}\label{eq:g-grad-bounds}
        \begin{split}
            & \|\nabla h_{\textup{LS}} - \widehat{\mathbf T}^{\textup{LS}}\|_{L^\infty([0,1]^d)} \le C_1 K^{-(\beta-1)},
            \\
            & \|\nabla^{2} h_{\textup{LS}} - \nabla \widehat{\mathbf T}^{\textup{LS}}\|_{L^\infty([0,1]^d)} \le C_2 K^{-(\beta-2)},
        \end{split}
    \end{equation}
    and hence
    \begin{equation}\label{eq:lip-final}
        \operatorname{Lip}(\nabla h_{\textup{LS}}) \le \lambda + c K^{-(\beta-2)}.
    \end{equation}
    Moreover, there exists a DNN with mixed activations (ReLU and ReQU), $\tilde{\mathbf{T}}_{\widehat{\phi}^{\textup{LS}}}:[0,1]^d\to\mathbb R^d,$ such that
    \begin{equation}\label{eq:exact-mixed-repu}
        \tilde{\mathbf{T}}_{\widehat{\phi}^{\textup{LS}}}
        = \nabla h_{\textup{LS}} \text{ on }[0,1]^d,
    \end{equation}
    with architecture
    $
    L' \le 3L+3,
    W' \le 6dW
    $
    and $\|\tilde{\mathbf{T}}_{\widehat{\phi}^{\textup{LS}}}\|_{L^\infty([0,1]^d)} < \infty.$
    In particular, we have
    \begin{equation}\label{eq:repu-comparator-exact}
        \|\tilde{\mathbf{T}}_{\widehat{\phi}^{\textup{LS}}}-\widehat{\mathbf T}^{\textup{LS}}\|_\infty \lesssim K^{-(\beta-1)},
        \quad
        \operatorname{Lip}(\tilde{\mathbf{T}}_{\widehat{\phi}^{\textup{LS}}}) \le \lambda + c K^{-(\beta-2)} .
    \end{equation}
\end{lemma}

\begin{proof}[Proof of \cref{lem:lipaware_requ_mixedrepu_exact}]
    From \cref{eq:pre-approx} with $\ell=1$ and the definitions $\nabla h_{\textup{LS}}$ and $\widehat{\mathbf T}^{\textup{LS}}=\nabla\hat\phi^{\textup{LS}}$,
    \begin{equation}
    \|\nabla h_{\textup{LS}} - \widehat{\mathbf T}^{\textup{LS}}\|_\infty
    = \|\nabla h_{\textup{LS}} - \nabla \hat\phi^{\textup{LS}}\|_\infty
    \le C_0 K^{-(\beta-1)} .
    \end{equation}
    Similarly, using \cref{eq:pre-approx} with $\ell=2$,
    \begin{equation}
    \|\nabla^{2} h_{\textup{LS}} - \nabla \widehat{\mathbf T}^{\textup{LS}}\|_\infty
    = \|\nabla^2 h_{\textup{LS}} - \nabla^2 \hat\phi^{\textup{LS}}\|_\infty
    \le C_0 K^{-(\beta-2)} .
    \end{equation}
    Thus \cref{eq:g-grad-bounds} holds with $C_1=C_2=C_0$. For Lipschitz constants,
    \begin{equation}
        \begin{split}
            \big|\operatorname{Lip}(\nabla h_{\textup{LS}})-\operatorname{Lip}(\widehat{\mathbf T}^{\textup{LS}})\big|
            & \le \sup_{x}\|\nabla^{2} h_{\textup{LS}}(x)-\nabla \widehat{\mathbf T}^{\textup{LS}}(x)\|_{\textup{op}}
            \\
            & \le d \|\nabla^{2} h_{\textup{LS}}-\nabla \widehat{\mathbf T}^{\textup{LS}}\|_\infty
            \le d C_2 K^{-(\beta-2)}.
        \end{split}
    \end{equation}
    Since $\hat\phi^{\textup{LS}}\in\mathcal{J}_{\lambda}$, its gradient $\widehat{\mathbf T}^{\textup{LS}}$ is $\lambda$-Lipschitz. Combining this with the Hessian approximation bound above gives $ \operatorname{Lip}(\tilde{\mathbf{T}}_{\widehat{\phi}^{\textup{LS}}}) \le \lambda + c K^{-(\beta-2)} $ with $c:=d C_2$.

    Finally, \cref{lem:repu-partial-ours} implies that there exists a DNN in which each hidden unit uses either the ReQU or a ReLU activation, $\tilde{\mathbf{T}}_{\widehat{\phi}^{\textup{LS}}}:[0,1]^d\to\mathbb R^d,$ such that $ \tilde{\mathbf{T}}_{\widehat{\phi}^{\textup{LS}}} = \nabla h_{\textup{LS}} \quad \text{on }[0,1]^d, $
    with architecture
    $ L' \le 3L+3, W' \le 6dW$
    and $\|\tilde{\mathbf{T}}_{\widehat{\phi}^{\textup{LS}}}\|_{L^\infty([0,1]^d)}=\|\nabla h_{\textup{LS}}\|_{L^\infty([0,1]^d)} < \infty.$
    This concludes the proof.
\end{proof}
\begin{lemma}\label{lem:smooth_size_schedule_n_rate}
    Fix $\beta>2$.
    Let $\bar{\mathbf T}^{\textup{sm}} = \tilde{\mathbf{T}}_{\widehat{\phi}^{\textup{LS}}}$ in \cref{lem:lipaware_requ_mixedrepu_exact} with $K = K_{\widetilde{n}}.$
    Choose the DNN parameters as
    $
    L \gtrsim \log d + \lfloor \beta \rfloor + \log \log H,
    W = W_{nm} \gtrsim (1 \vee d) (K_{\widetilde{n}} + \lfloor \beta \rfloor)^{d}.
    $
    If $K_{\widetilde{n}} \asymp \widetilde{n}^{\frac{1}{d(\beta - 1)}}$, then we have
    $ \|\bar{\mathbf T}^{\textup{sm}}-\widehat{\mathbf T}^{\textup{LS}}\|_\infty^2 \lesssim \widetilde{n}^{-2/d} $
    and
    $ \operatorname{Lip}(\bar{\mathbf T}^{\textup{sm}}) \le \lambda + c K_{\widetilde{n}}^{-(\beta-2)}.
    $
\end{lemma}

\begin{proof}[Proof of \cref{lem:smooth_size_schedule_n_rate}]
    By \cref{lem:lipaware_requ_mixedrepu_exact} with $K=K_{\widetilde{n}}$, we have
    $ \|\bar{\mathbf T}^{\textup{sm}}-\widehat{\mathbf T}^{\textup{LS}}\|_\infty \lesssim K_{\widetilde{n}}^{-(\beta-1)} . $
    Choosing the parameter of the ReQU DNN $h_{\textup{LS}}$ in \cref{lem:requ_simul_correct} as:
    \begin{equation}\label{eq:network_sizes_append}
        \begin{split}
            & L \gtrsim \log d+\lfloor\beta\rfloor+\log\log H,
            \\
            & W_{nm} \gtrsim (1\vee d) \big(K_{\widetilde{n}}+\lfloor\beta\rfloor\big)^{d}.
        \end{split}
    \end{equation}
    Let $(K_{\widetilde{n}})^d \asymp \widetilde{n}^{1/(\beta-1)}.$
    Then, we have
    \begin{equation}
        \|\bar{\mathbf T}^{\textup{sm}}-\widehat{\mathbf T}^{\textup{LS}}\|_\infty^2 \lesssim K_{\widetilde{n}}^{-2(\beta-1)}
        = \Big(\widetilde{n}^{\frac{1}{d(\beta-1)}}\Big)^{-2(\beta-1)} = \widetilde{n}^{-2/d}
    \end{equation}
    and
    \begin{equation}\label{eq:lip_budget_schedule}
    \operatorname{Lip}(\bar{\mathbf T}^{\textup{sm}}) \le \lambda + c K_{\widetilde{n}}^{-(\beta-2)}.
    \end{equation}
\end{proof}

\begin{lemma}[Lipschitz-aware DNN comparator]\label{lem:direct_lipschitz_comparator}
    Let $K\ge2$ be an integer and let $\mathbf T:\Omega\to\Omega$ be $\lambda$-Lipschitz.
    There exists a DNN $\bar{\mathbf T}_K$ with mixed ReLU/ReQU activations, depth $L_d$, width at most $C_d(K+1)^{d+1}$, and all parameters in $[-1,1]$, such that
    $$
    \|\bar{\mathbf T}_K-\mathbf T\|_\infty\le d\lambda/K,
    \quad \operatorname{Lip}(\bar{\mathbf T}_K)\le d\lambda.
    $$
    The constants depend only on $d$ and $\lambda$.
\end{lemma}
\begin{proof}
    Extend each coordinate of $\mathbf T$ to $[0,1]^d$ by
    $$
    g_\ell(x)=\min\{1,\max\{0,\inf_{z\in\Omega}[T_\ell(z)+\lambda\|x-z\|]\}\},\quad \ell\in[d].
    $$
    Each $g_\ell$ is $\lambda$-Lipschitz and agrees with $T_\ell$ on $\Omega$.
    For $j=0,\ldots,K$, let $h_j(t)=(1-K|t-j/K|)_+$ on $[0,1]$, and define
    $$
    \bar{\mathbf T}_K(x)=\sum_{\mathbf j\in\{0,\ldots,K\}^d}\mathbf g(\mathbf j/K)\prod_{s=1}^d h_{j_s}(x_s).
    $$
    The product weights are nonnegative, sum to one, and use grid points within distance $\sqrt d/K$ of $x$.
    Hence each coordinate has approximation error at most $\lambda\sqrt d/K$, giving the stated bound.
    Within each grid cell, each Jacobian entry is a convex combination of coordinate edge differences divided by $1/K$, and is bounded in absolute value by $\lambda$.
    Thus its operator norm is at most $d\lambda$, and continuity across the cells gives the Lipschitz bound on $[0,1]^d$.

    The interpolant is exactly realizable by a bounded-parameter mixed ReLU/ReQU network.
    ReLU computes $|t-j/K|$; the factor $K$ is obtained by summing $K$ identical copies with unit weights.
    Products of nonnegative values are computed using
    $$
    uv=\textup{ReQU}((u+v)/2)-\textup{ReQU}((u-v)/2)-\textup{ReQU}((v-u)/2), \quad u,v\in[0,1].
    $$
    A binary tree of these multipliers has depth depending only on $d$, and parallelizing the grid terms gives width at most $C_d(K+1)^{d+1}$.
    The final coefficients $g_\ell(\mathbf j/K)$ lie in $[0,1]$, so all parameters can be taken in $[-1,1]$.
    Precomposing with the $1$-Lipschitz clipping map $\pi(x)_s=\operatorname{ReLU}(x_s)-\operatorname{ReLU}(x_s-1)$ gives the same global Lipschitz bound without changing the approximation on $\Omega$; its constant additional depth and width are absorbed into $L_d$ and $C_d$.
\end{proof}

\begin{lemma}\label{lem:size_schedule_n_rate}
    Let $\bar{\mathbf T}=\bar{\mathbf T}_{K_{\widetilde n}}$ be obtained by applying
    \cref{lem:direct_lipschitz_comparator} to $\mathbf T=\mathbf T_0$.
    Choose
    $
    K_{\widetilde{n}}=\max\{2,\lceil {\widetilde{n}}^{1/d}\rceil\},
    L=L_d,
    $
    and
    $ W_{nm}\ge C_d(K_{\widetilde{n}}+1)^{d+1}. $
    Then, we have
    $$
    \|\bar{\mathbf T}-\mathbf T_0\|_\infty^2
    \lesssim {\widetilde{n}}^{-2/d},
    \quad
    \operatorname{Lip}(\bar{\mathbf T})\le d\lambda.
    $$
\end{lemma}

\begin{proof}
    By \cref{lem:direct_lipschitz_comparator},
    $$
    \|\bar{\mathbf T}-\mathbf T_0\|_\infty^2
    \le \frac{d^2\lambda^2}{K_{\widetilde{n}}^2}
    \lesssim {\widetilde{n}}^{-2/d}.
    $$
    Moreover,
    $\operatorname{Lip}(\bar{\mathbf T})\le d\lambda<\lambda+c'$.
\end{proof}

\paragraph{Proof of \cref{thm:dnn_main}.}\label{sec:appen-full_proof}

\begin{proof}[Proof of \cref{thm:dnn_main}]
    Write $\widehat{\mathbf T}=\widehat{\mathbf T}_{nm}^{\textup{DNN}}$ for brevity, and take expectation with respect to the data and to OT plan $\hat\Gamma=[\hat\gamma_{ij}]$.
    We also note that the OT plan naturally induces the weighted regression objective
    $ \sum_{i,j}\widehat{\gamma}_{ij}\|Y_j-\mathbf{T}(X_i)\|^2, $
    which is equivalent to the standard objective in \cref{def:real_est} up to a constant independent of $\mathbf{T}.$
    Thus, in this proof, we work with the weighted objective $ \sum_{i,j}\widehat{\gamma}_{ij}\|Y_j-\mathbf{T}(X_i)\|^2. $
    First, we decompose the error using the Voronoi cells.
    Let $V_i=\{x\in\Omega:\|x-X_i\|\le \|x-X_k\|\ \forall k\neq i\}, i \in [n]$ be the Voronoi cells induced by $X_1,\dots,X_n$.
    Then, we can decompose the error as
    \begin{equation}
        \|\widehat{\mathbf T}-\mathbf T_0\|_{L^2(P)}^2
        =\sum_{i=1}^n\int_{V_i}\|\widehat{\mathbf T}(x)-\mathbf T_0(x)\|^2  dP(x).
    \end{equation}
    For every $x\in V_i$, by $(a+b+c)^2\le 3(a^2+b^2+c^2)$ and the triangle inequality, we have
    $ \|\widehat{\mathbf T}(x)-\mathbf T_0(x)\|^2
    \le 3\|\widehat{\mathbf T}(x)-\widehat{\mathbf T}(X_i)\|^2
    +3\|\widehat{\mathbf T}(X_i)-\mathbf T_0(X_i)\|^2
    +3\|\mathbf T_0(X_i)-\mathbf T_0(x)\|^2. $
    Integrating over $V_i$ and summing over $i$ yields
    \begin{equation}\label{eq:basic-split}
        \begin{split}
            \|\widehat{\mathbf T}-\mathbf T_0\|_{L^2(P)}^2
            & \le 3\sum_{i=1}^n\int_{V_i}\|\widehat{\mathbf T}(x)-\widehat{\mathbf T}(X_i)\|^2  dP(x)
            +3\sum_{i=1}^n\int_{V_i}\|\mathbf T_0(x)-\mathbf T_0(X_i)\|^2  dP(x)
            \\
            & +3\sum_{i=1}^n P(V_i)\|\widehat{\mathbf T}(X_i)-\mathbf T_0(X_i)\|^2.
        \end{split}
    \end{equation}
    Second, we derive upper bounds for the first and second terms of the right-hand-side in \cref{eq:basic-split}.
    By construction $\widehat{\mathbf T}\in \Phi_{nm}^{\textup{DNN}}$ so that $\operatorname{Lip}(\widehat{\mathbf T})\le \lambda+c'$, and by \cref{assumption-3}, $\operatorname{Lip}(\mathbf T_0)\le \lambda$.
    Applying \cref{lem:lambda_lipschitz_partition} with $f=\widehat{\mathbf T}$ and with $f=\mathbf T_0$, and then taking expectations in \cref{eq:basic-split}, we obtain
        \begin{equation}\label{eq:pop-to-emp}
            \mathbb E\|\widehat{\mathbf T}-\mathbf T_0\|_{L^2(P)}^2
            \lesssim \Big(\frac{\log n}{n}\Big)^{2/d}
            +\mathbb E \left(
            n\Big(\max_{1\le i\le n}P(V_i)\Big)\cdot \|\widehat{\mathbf T}-\mathbf T_0\|_{L^2(P_n)}^2
            \right),
        \end{equation}
    where $ \|\widehat{\mathbf T}-\mathbf T_0\|_{L^2(P_n)}^2 := \frac1n\sum_{i=1}^n\| \widehat{\mathbf T}(X_i) - \mathbf{T}_{0}(X_i) \|^2$.
    Third, we derive the bound for the third term.
    For any vectors $a,b,c$, it holds that $\|a-b\|^2 \le 2\|a-c\|^2+2\|c-b\|^2$.
    Hence, for all $i,j$,
    \begin{equation}
        \|\widehat{\mathbf T}(X_i)-\mathbf T_0(X_i)\|^2
        \le 2\|\widehat{\mathbf T}(X_i)-Y_j\|^2 + 2\|Y_j-\mathbf T_0(X_i)\|^2.
    \end{equation}
    Since each row sum of $\hat\Gamma$ is $1/n$, summing against $\hat\gamma_{ij}$ gives the following bound.
    \begin{equation}\label{eq:emp-split}
        \begin{split}
            \|\widehat{\mathbf T}-\mathbf T_0\|_{L^2(P_n)}^2
            & =\sum_{i,j}\hat\gamma_{ij}\|\widehat{\mathbf T}(X_i)-\mathbf T_0(X_i)\|^2
            \\
            & \le 2\sum_{i,j}\hat\gamma_{ij}\|\widehat{\mathbf T}(X_i)-Y_j\|^2
            +2\sum_{i,j}\hat\gamma_{ij}\|Y_j-\mathbf T_0(X_i)\|^2.
        \end{split}
    \end{equation}
    Let $\bar{\mathbf T}$ be the DNN approximator of $\mathbf T_0$ from \cref{lem:size_schedule_n_rate}.
    It satisfies $\operatorname{Lip}(\bar{\mathbf T})\le d\lambda<\lambda+c'$, so $\bar{\mathbf T}\in\Phi_{nm}^{\textup{DNN}}$.
    Hence, by optimality of $\widehat{\mathbf T}$ over $\Phi_{nm}^{\textup{DNN}}$,
    $
    \sum_{i,j}\hat\gamma_{ij}\|\widehat{\mathbf T}(X_i)-Y_j\|^2
    \le \sum_{i,j}\hat\gamma_{ij}\|\bar{\mathbf T}(X_i)-Y_j\|^2.
    $
    Expanding $\bar{\mathbf T}$ around $\mathbf T_0$ gives, for each $(i,j)$,
    $\|\bar{\mathbf T}(X_i)-Y_j\|^2\le2\|\bar{\mathbf T}(X_i)-\mathbf T_0(X_i)\|^2+2\|\mathbf T_0(X_i)-Y_j\|^2,$
    and summing over $\hat\gamma_{ij}$ and using the fact that the total mass is $1$, we get
    \begin{equation}\label{eq:emp-bound}
    \sum_{i,j}\hat\gamma_{ij}\|\bar{\mathbf T}(X_i)-Y_j\|^2
    \le2d\|\bar{\mathbf T}-\mathbf T_0\|_\infty^2
    +2\sum_{i,j}\hat\gamma_{ij}\|\mathbf T_0(X_i)-Y_j\|^2.
    \end{equation}
    Combining \cref{eq:emp-split}-\cref{eq:emp-bound}, we obtain
    \begin{equation}\label{eq:emp-final}
    \|\widehat{\mathbf T}-\mathbf T_0\|_{L^2(P_n)}^2
    \lesssim\|\bar{\mathbf T}-\mathbf T_0\|_\infty^2
    +\sum_{i,j}\hat\gamma_{ij}\|\mathbf T_0(X_i)-Y_j\|^2.
    \end{equation}
    When, in addition, $\phi_0\in\mathcal C^\beta([0,1]^d,H)$, the comparator can also be built from the least-squares map $\widehat{\mathbf T}^{\textup{LS}}=\nabla\hat\phi^{\textup{LS}}$ instead of $\mathbf T_0$.
    In this case, $\phi_0\in\mathcal J_\lambda\cap\mathcal C^\beta([0,1]^d,H)$, so the optimality of $\hat\phi^{\textup{LS}}$ gives $\sum_{i,j}\hat\gamma_{ij}\|\widehat{\mathbf T}^{\textup{LS}}(X_i)-Y_j\|^2\le\sum_{i,j}\hat\gamma_{ij}\|\mathbf T_0(X_i)-Y_j\|^2$.
    Moreover, $\bar{\mathbf T}^{\textup{sm}}$ in \cref{lem:smooth_size_schedule_n_rate} satisfies $\|\bar{\mathbf T}^{\textup{sm}}-\widehat{\mathbf T}^{\textup{LS}}\|_\infty^2\lesssim\widetilde n^{-2/d}$ and $\operatorname{Lip}(\bar{\mathbf T}^{\textup{sm}})\le\lambda+c'$ for all sufficiently large $\widetilde n$, so it belongs to $\Phi_{nm}^{\textup{DNN}}$ after enlarging its constant depth and its width, where each parameter of magnitude larger than one is split into parallel unit-weight copies as in the proof of \cref{lem:direct_lipschitz_comparator}, if necessary.
    Replacing $\bar{\mathbf T}$ and $\mathbf T_0$ by $\bar{\mathbf T}^{\textup{sm}}$ and $\widehat{\mathbf T}^{\textup{LS}}$ in \cref{eq:emp-bound} then gives \cref{eq:emp-final} with $\|\bar{\mathbf T}^{\textup{sm}}-\widehat{\mathbf T}^{\textup{LS}}\|_\infty^2$ in place of $\|\bar{\mathbf T}-\mathbf T_0\|_\infty^2$, which is of the same order.

    Then, we aggregate the bounds as follows.
    Let
    $$
    M_n:=\max_{1\le i\le n}P(V_i),
    \quad
    Z_{nm}:=
    \sum_{i,j}\hat\gamma_{ij}
    \|\mathbf T_0(X_i)-Y_j\|^2.
    $$
    By \cref{lem:voronoi_mass_diam}(i), taking $\delta=n^{-2}$ and using $M_n\le1$, we have $\mathbb E (nM_n) \lesssim\log n+n\delta\lesssim\log n,$ for all $n \ge 2.$
    Plugging \cref{eq:emp-final} into \cref{eq:pop-to-emp}
    and taking expectations, we obtain
    \begin{equation}\label{eq:aggregate-bound}
        \begin{split}
        \mathbb E\|\widehat{\mathbf T}-\mathbf T_0\|_{L^2(P)}^2
        \lesssim
        \Big(\frac{\log n}{n}\Big)^{2/d}
        +
        \|\bar{\mathbf T}-\mathbf T_0\|_\infty^2
        \mathbb E[nM_n]
        +
        \mathbb E[nM_nZ_{nm}].
        \end{split}
    \end{equation}
    Finally, we construct the DNN architecture to approximate $\mathbf T_0$.
    By \cref{lem:size_schedule_n_rate},
    with
    $
    K_{\widetilde{n}}=\max\{2,\lceil {\widetilde{n}}^{1/d}\rceil\},
    $
    we have
    $$
    \|\bar{\mathbf T}-\mathbf T_0\|_\infty^2
    \lesssim {\widetilde{n}}^{-2/d},
    \quad
    \operatorname{Lip}(\bar{\mathbf T})\le d\lambda.
    $$
    Hence $\bar{\mathbf T}$ belongs to the admissible Lipschitz class used to define $\widehat{\mathbf T}$.
    Combining \cref{eq:aggregate-bound,eq:weighted-residual-rate} with the bound $\mathbb E[nM_n]\lesssim\log n$ above and
    \cref{lem:size_schedule_n_rate}, we obtain
    \begin{equation}
        \begin{split}
            \mathbb E
            \|\widehat{\mathbf T}-\mathbf T_0\|_{L^2(P)}^2
            \lesssim
            \Big(\frac{\log n}{n}\Big)^{2/d}
            +{\widetilde{n}}^{-2/d}\log n
            +
            \begin{cases}
                \widetilde n^{-1}\log n, & d=1,\\
                \widetilde n^{-1}(\log\widetilde n)^2\log n, & d=2,\\
                \widetilde n^{-2/d}\log n, & d\ge3,
            \end{cases}
        \end{split}
    \end{equation}
    achieving the minimax convergence rate up to logarithmic factors.
    This bound carries an additional $\log n$ factor for $d=2$ compared with \cref{thm:dnn_main}.
    To obtain the stated rates for every $d$, let $\pi_n$ be an optimal coupling of $P$ and $P_n$.
    Since $\widehat{\mathbf T}-\mathbf T_0$ is $(2\lambda+c')$-Lipschitz, integrating over $\pi_n$ gives, for every realization of the data,
    \begin{equation}\label{eq:coupling-population-bound}
        \|\widehat{\mathbf T}-\mathbf T_0\|_{L^2(P)}^2
        \le 2\|\widehat{\mathbf T}-\mathbf T_0\|_{L^2(P_n)}^2
        +2(2\lambda+c')^2 W_2^2(P,P_n).
    \end{equation}
    Corollary 8 of \cite{manole2024pluginestimationsmoothoptimal}, applied to the identity map transporting $P$ to itself, gives $\mathbb E W_2^2(P,P_n)\lesssim\bar\kappa_n$.
    Combining this with \cref{eq:emp-final,lem:prop_14_manole,lem:size_schedule_n_rate}, we obtain
    \begin{equation}
        \mathbb E\|\widehat{\mathbf T}-\mathbf T_0\|_{L^2(P)}^2
        \lesssim \widetilde n^{-2/d}+\bar\kappa_n+\bar\kappa_{\widetilde n}
        \lesssim\bar\kappa_{\widetilde n},
    \end{equation}
    which implies the stated bounds in \cref{thm:dnn_main}.
    This completes the proof.
\end{proof}

\clearpage
\section{Experiments}\label{sec:appen-exps}

\subsection{Implementation Details}\label{sec:appen-simulation-tuning}

In this section, we describe the implementation details such as hyperparameter selection and model architectures used in our experiments.
All experiments are run on several NVIDIA RTX 3090 GPUs.

\paragraph{Overall protocol.}

(i) Hyperparameter search spaces are documented below for each method and are chosen by combining each paper's reported settings and a two to three times expansion.
(ii) Each method is trained with up to $10{,}000$ (outer) gradient steps, when applicable, to provide sufficient time for convergence.
(iii) Data are scaled to have zero mean and unit variance, when applicable.
All the methods are trained on the normalized inputs, and the transported inputs are denormalized back to the original target scale when we calculate the performance metrics.
Note that this normalization particularly stabilizes the min-max learning methods such as ICNN.

\paragraph{BROT.}

We use the \texttt{emd} module of the POT library \cite{flamary2021pot} to solve the unregularized empirical Kantorovich problem.
Among the training iterations, we pick the hyperparameter with the lowest validation \texttt{Wass} (or task-appropriate validation \texttt{MMD} for the single-cell experiments).
For large-scale experiments, we use a chunked version of Step 1.
We randomly partition the source and target data into paired chunks, solve an unregularized OT problem within each chunk pair, and collect the resulting barycentric targets.
For the domain adaptation experiments, we use a source chunk size of $2{,}048$, with target chunk sizes proportional to the target/source sample size ratio.
The downstream classifier is trained for up to $80$ epochs on the transported source data with an $80{:}20$ training/validation split, and we select the model at the iteration with the highest validation accuracy.
We then evaluate the model on the test data.

The search spaces are:
the activation in $\{\textup{ReLU}, \textup{ReQU}, \textup{GELU}\}$,
the hidden widths among a few MLP architectures from $4$ to $6$ layers wide between $64$ and $2{,}048$ units,
the optimizer (SGD with momentum or Adam),
the learning rate in $\{10^{-3}, 5\cdot 10^{-3}, 10^{-2}, 5\cdot 10^{-2}\}$,
the batch size in $\{256, 512, 1{,}024, 2{,}048\}$,
the number of epochs in $\{100, 500, 1{,}000, 2{,}000\}$,
and the gradient-penalty weight $\{0, 10^{-4}, 10^{-3}, 10^{-2}, 10^{-1}\}.$

\paragraph{1NN \cite{manole2024pluginestimationsmoothoptimal}.}

1NN solves the unregularized empirical Kantorovich problem and uses the resulting plan to assign each training data point $X_i$ to its barycentric target.
At test time, it extends the plan to unseen source data by the nearest neighbor interpolation using the training data.

\paragraph{LSOT \cite{seguy2018large}.}

LSOT solves a regularized dual Kantorovich problem (with either an $L_{2}$ or entropic regularizer) by stochastic optimization, and then trains a DNN by least-squares regression to fit the regularized barycentric targets.
We reimplement the dual SGD procedure, Algorithm 1 of \cite{seguy2018large}, and reuse the same MLP architecture as BROT on every dataset.
For each experiment, we pick the hyperparameter with the lowest validation \texttt{Wass} (or task-appropriate validation \texttt{MMD} for the single-cell experiments) over a grid search.
Because LSOT uses a regularized OT plan, a nonzero regularization strength may induce a bias relative to the unregularized OT map.
We therefore search the regularization weight down to the smallest value for which the optimization remains numerically stable.
The grid below starts at $0.001$, and we kept the lowest value for which the optimization of LSOT remains numerically stable.

The search spaces are as follows:
the regularizer type (either $L_{2}$ or entropic) and its strength in $\{0.001, 0.0025, 0.005, 0.01, 0.02, 0.05, 0.1\}$,
the potential learning rate in $\{5\cdot 10^{-4}, 10^{-3}, 5\cdot 10^{-3}, 10^{-2}\}$,
the potential optimizer (SGD or Adam),
the potential batch size in $\{50, 100, 200\}$,
the potential iteration number in $\{2{,}000, 5{,}000, 10{,}000\}$,
the map learning rate in $\{10^{-4}, 5\cdot 10^{-4}, 10^{-3}, 5\cdot 10^{-3}\}$,
the map batch size in $\{50, 100, 200\}$,
and the map iteration number in $\{2{,}000, 5{,}000, 10{,}000\}.$

\paragraph{ICNN \cite{pmlr-v119-makkuva20a}.}

ICNN parameterizes the potential $\phi$ as an input-convex neural network (ICNN) and trains it together with its conjugate by the min-max dual objective.
At inference time, we use the gradient of the trained potential as the estimate of the OT map.
We follow the parameterization of the CellOT codebase \cite{bunne2023learning} due to stability compared to the official codebase, with a linear output head and the uniform initialization on $[0, 0.1]$ for the convex weights.
The min-max objective is optimized by alternating: 1 map update per 10 potential updates with a soft convexity penalty of weight $\lambda_{\textup{cvx}}$ on the parameters.
This penalty is the squared Frobenius norm of the negative parts of the convex weights summed over all layers and scaled by $\lambda_{\textup{cvx}}$, which discourages negative weights and thus encourages the network to remain input-convex without requiring hard projection at every step.

For each experiment, we select the hyperparameter setting that minimizes the validation \texttt{Wass} (or the task-appropriate validation loss for the single-cell and domain adaptation experiments) over a random search.
The search space follows the CellOT codebase \cite{bunne2023learning} with the architecture (a $4$-layer or $3$-layer ICNN with widths in $\{64, 256, 512\}$), the leaky-ReLU activation, the convexity penalty weight $\lambda_{\textup{cvx}} \in \{0.5, 1.0, 2.0, 5.0\}$,
the map learning rate in $\{5\cdot 10^{-5}, 10^{-4}, 10^{-3}\}$,
the potential learning rate in $\{5\cdot 10^{-5}, 10^{-4}, 5\cdot 10^{-4}, 10^{-3}\}$,
the potential updates per map update in $\{5, 10, 16\}$,
the iteration budget in $\{2{,}000, 3{,}000, 5{,}000, 10{,}000\}$,
the batch size in $\{256, 512\}$, the weight decay in $\{0, 10^{-5}, 10^{-4}\}$,
and the convex-weight initialization in $\{\textup{uniform-small}, \textup{trunc-inv-sqrt}\}$.
On \textsc{AFHQv2}, the training diverges across every configuration in this search space, marked $\ast$ in \cref{tab:distmatch}.

\paragraph{DIOTM \cite{choi2025improving}.}

DIOTM builds on the semi-dual neural OT (SNOT) problem $\inf_{\mathbf{T}} \sup_{\psi} ( \mathbb{E}_{P}[\|x-\mathbf{T}(x)\|^{2} - \psi(\mathbf{T}(x))] + \mathbb{E}_{Q}[\psi( y)] )$ of \cite{fan2023neural,pmlr-v119-makkuva20a,rout2022generative}. It trains forward and backward transport maps jointly with a time-dependent potential along the displacement interpolation between the source and target distributions, and adds a Hamilton-Jacobi-Bellman (HJB) regularizer on the potential. Its hyperparameters include the divergence type in the potential loss, the cost coefficient $\tau$, and the HJB regularization weight $\lambda_{\textup{reg}}.$
For each experiment we select the hyperparameter setting that minimizes the validation \texttt{Wass} (or task-appropriate validation \texttt{MMD} for the single-cell experiments).

The search spaces are:
the divergence type in $\{\textup{linear}, \textup{KL}, \chi^{2}, \textup{softplus}\}$,
the cost coefficient $\tau \in \{0.01, 0.1, 1.0\}$,
the HJB regularization weight $\lambda_{\textup{reg}} \in \{0, 10^{-3}, 10^{-2}, 10^{-1}, 1\}$,
the dual learning rate in $\{10^{-5}, 10^{-4}, 5\cdot 10^{-4}\}$,
the map learning rate in $\{10^{-5}, 10^{-4}, 5\cdot 10^{-4}, 10^{-3}\}$,
the number of dual inner steps in $\{3, 5, 8, 10\}$ with map inner steps in $\{1, 2, 4\}$,
the iteration number in $\{2{,}000, 5{,}000, 10{,}000, 20{,}000\}$,
the hidden width in $\{64, 128, 256, 512\}$,
and the batch size in $\{128, 200, 256, 512, 1{,}024\}$.

\paragraph{OTP \cite{choi2025overcoming}.}

OTP adds a Gaussian-smoothing schedule on top of the SNOT problem \cite{fan2023neural,pmlr-v119-makkuva20a,rout2022generative}.
At each training step, each source data point is perturbed by Gaussian noise whose scale is annealed from $\xi_{\max}$ to $\xi_{\min}$ along a linear or cosine schedule, either additively or in a variance-preserving form.
In the additive form the perturbed source is $\widetilde{X}_{i, t} = X_{i} + \xi_{t} \boldsymbol{\epsilon}$, $\boldsymbol{\epsilon} \sim \mathcal{N}(\boldsymbol{0}, I)$.
In the variance-preserving form it is $\widetilde{X}_{i, t} = \sqrt{1-\xi_{t}^{2}} X_{i} + \xi_{t} \boldsymbol{\epsilon}$, which keeps the marginal variance of $\widetilde{X}_{i, t}$ equal to that of $X_{i}$.
The noise scale $\xi_{t}$ decays from $\xi_{\max}$ at the first iteration to $\xi_{\min}$ at the last iteration along the chosen schedule.
For each experiment we select the hyperparameter setting that minimizes the validation \texttt{Wass} (or the task-appropriate validation \texttt{MMD} for the single-cell experiments) over a grid search.
We implement OTP by reusing the DIOTM codebase and adding the smoothing of the source data described above.
The search space therefore inherits the DIOTM hyperparameters provided above.
We additionally do the grid search: the noise mode in $\{\textup{additive}, \textup{variance-preserving}\}$, the noise schedule in $\{\textup{linear}, \textup{cosine}\}$, the maximum noise scale $\xi_{\max} \in \{0.1, 0.3, 0.5, 1.0\}$, and the minimum noise scale $\xi_{\min} \in \{0, 0.01\}$.

\paragraph{MM-v2 and W2 \cite{korotin2021neural}.}

MM-v2 uses two input-convex neural networks as potentials with alternating maximin updates.
W2 uses a non-maximin objective with a cycle-consistency penalty.
We follow the implementations of \cite{korotin2021neural} and select hyperparameters similarly to BROT.
For the map quality experiment, both MM-v2 and W2 solvers use widths $(64,64,32)$, batch size $1{,}024$ and learning rate $10^{-3}$.
On AFHQ, we use widths $(256,256)$, batch size $128$ and learning rate $10^{-3}$.

\paragraph{Entropic map estimator (ENT) \cite{pooladian2021entropic}.}

ENT uses the dual potential of entropically regularized OT to extend its barycentric map to unseen source data.
Its regularization strength is $\varepsilon=\rho\overline c>0$, where $\rho$ is a dimensionless scale and $\overline c$ is the mean pairwise quadratic cost on the training data.
We select one from $\rho\in\{0.001,0.003,0.01,0.03,0.1,0.3\}$ using the same validation metric as BROT (and the neural estimators).
In the experiment of \cref{sec:rate}, $\varepsilon$ is instead selected as described in \cref{sec:appen-entropic-cost}.

\clearpage
\subsection{Statistical Convergence to the Ground-Truth Optimal Transport Map}\label{sec:appen-convergence_nonlinear}

We provide \cref{fig:appen-convergence-sem}, a copy of \cref{fig:convergence_both}(a) with the standard error bands to show the statistical uncertainty of this experiment.

\begin{figure}[h!]
    \centering
    \includegraphics[width=0.41\linewidth]{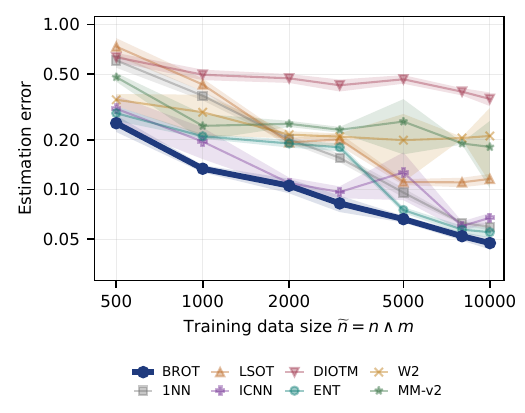}
    \caption{\textbf{Statistical convergence to the ground-truth OT map.}
    A copy of \cref{fig:convergence_both}(a) with the standard error bands.
    }
    \label{fig:appen-convergence-sem}
\end{figure}

We further give details of the synthetic dataset used in \cref{sec:rate}.
Fix $R > 0$ and coefficients $a_{s} > 0$, $b_{s} \ge 0$ for $s=1,\dots,d$.
Define the potential and candidate OT map for $ x = [x_{1}, \ldots, x_{d}]^{\top} \in \mathbb{R}^{d}$ as:
\begin{equation}\label{eq:appen-nonlinear-potential}
    \phi( x) := \frac{1}{2}\sum_{s=1}^{d} a_{s} x_{s}^2 + \frac{1}{4}\sum_{s=1}^{d} b_{s} x_{s}^4,
    \quad
    \mathbf{T}_0( x) := \nabla \phi(x) = \bigl(a_{s} x_{s} + b_{s} x_{s}^3\bigr)_{s=1}^{d},
\end{equation}
and let $P := \mathcal{N}(0, I_d)|_{R}$ be the standard Gaussian truncated to the ball of radius $R.$
The target distribution is $Q := \mathbf{T}_{0\#} P$.
In our experiments we use $d=2, (a_1, a_2)=(1.5, 0.8), (b_1, b_2)=(0.3, 0.1)$, and $R=3.5$.
\cref{fig:appen-PQ_viz} visualizes empirical distributions of $P$ and $Q$ used in this analysis, and \cref{prop:appen-nonlinear-ot} shows that the construction in \cref{eq:appen-nonlinear-potential} yields a Lipschitz OT map.

\begin{figure}[h!]
    \vskip -0.1in
    \centering
    \includegraphics[width=0.46\linewidth]{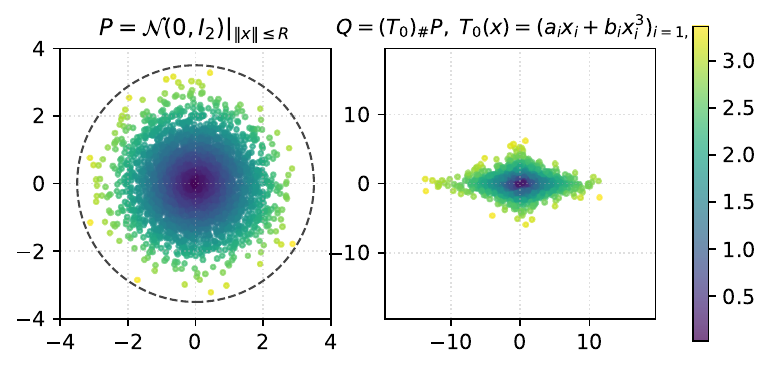}
    \caption{
    \textbf{Visualizations of the source $P$ (left) and target $Q = \mathbf{T}_{0\#}P$ (right) distributions used in \cref{sec:rate}.}
    The color of each point indicates the radius of data $\|  x \|.$
    }
    \label{fig:appen-PQ_viz}
\end{figure}

\begin{proposition}[Lipschitz OT map from $P$ to $Q$]\label{prop:appen-nonlinear-ot}
    Let $\phi$, $\mathbf{T}_0$, $P$, $Q$ be as in \cref{eq:appen-nonlinear-potential} with $a_{s} > 0$, $b_{s} \ge 0$, and $R<\infty$.
    Then, $\mathbf{T}_0$ is the unique OT map from $P$ to $Q,$ and is Lipschitz.
\end{proposition}

\begin{proof}

    Since $\phi$ is separable across coordinates, its Hessian is diagonal.
    That is, $ \nabla^2\phi( x) = \operatorname{diag} \bigl(a_{s} + 3b_{s} x_{s}^2\bigr)_{s=1}^{d}, $
    which is strictly positive because $a_{s} > 0$ and $b_{s} \ge 0$.
    Hence $\phi$ is smooth and strictly convex, and lower-semicontinuous on $\mathbb{R}^d$.
    At the same time, $P$ admits a density supported on the compact ball $\{\|x\|\le R\}$ and hence has finite second moment.
    $\mathbf{T}_0$ is also smooth on this ball, since its outputs are in $\{|y_{s}| \le a_{s} R + b_{s} R^3\}$, so $Q=\mathbf{T}_{0\#}P$ also has finite second moment.

    By the Knott-Smith theorem (Theorem 2.12 of \cite{villani2003topics}), if $\phi:\mathbb{R}^d\to\mathbb{R}\cup\{+\infty\}$ is convex and lower-semicontinuous with $\nabla\phi$, and if both $P$ and $(\nabla\phi)_{\#}P$ have finite second moment, then $\nabla\phi$ is the unique OT map from $P$ to $(\nabla\phi)_{\#}P$ for the quadratic cost.
    The above two steps satisfy these conditions, so $\mathbf{T}_0=\nabla\phi$ is the optimal transport map from $P$ to $Q$.
    On the support $\{\|x\|\le R\},$ the diagonal Hessian satisfies $\nabla^2\phi(x)\preceq \operatorname{diag} \bigl(a_{s} + 3b_{s} R^2\bigr)_{s=1}^{d}$, so $\mathbf{T}_0$ is $L_0$-Lipschitz on this set with Lipschitz constant
    $ L_0 = \max_{s\in[d]} \bigl(a_{s} + 3b_{s} R^2\bigr). $
    For $d=2$, our setting of $(a_1, a_2)=(1.5, 0.8)$, $(b_1, b_2)=(0.3, 0.1)$, $R=3.5$ gives $L_0\approx 12.52$.
    Hence, $\mathbf{T}_0$ is a Lipschitz OT map.

\end{proof}

\paragraph{Sensitivity to DNN architecture and Jacobian penalty.}\label{sec:appen-sensitivity}

The theoretically designed DNN architecture in \cref{thm:dnn_main} depends on the sample sizes $n$ and $m$, whereas our experiments in \cref{sec:exp_cdot} use a fixed MLP architecture with the Jacobian penalty.
The sample size-dependent architecture is mainly a theoretical tool for controlling approximation error, as done in DNN-based nonparametric regression literature \cite{YAROTSKY2017103,Schmidt-Hieber20201875,suzuki2018adaptivity,BELOMESTNY2023242,shen2024differentiableneuralnetworksrepu}, so \cref{thm:dnn_main} can be interpreted as an existence result rather than as a practical architectural prescription.
In particular, the theoretically required width can be large even at moderate dimensions, and for $d>2$, the rate $\widetilde n^{-2/d}$ reflects the curse of dimensionality.
Convergence rates depending on the intrinsic dimension may be possible under additional low-dimensional structure \cite{jiao2023deep}.
Furthermore, the Lipschitz constraint in \cref{thm:dnn_main} is practically enforced by a Jacobian penalty in our experiments, which encourages using the Lipschitz continuity.

Therefore, we empirically assess the numerical sensitivity of BROT to architecture choice and the Jacobian penalty weight on the synthetic dataset used in \cref{sec:rate}.
For $n \in \{500,1000,2000,3000,5000,8000,10000\},$ we compare (i) three architectures with the Jacobian penalty weight fixed to $10^{-2}$ and (ii) three Jacobian penalty weights in $\{10^{-4},10^{-3},10^{-2}\}$ with the fixed architecture used for the main results in \cref{thm:dnn_main}, under five random seeds.
As shown in \cref{fig:appen-arch-robustness}, the estimation errors are broadly similar across the three architectures and the penalty weights, although the largest penalty weight ($10^{-2}$) leads to a moderately larger error at large $n$.
Overall, BROT performs reasonably well across these choices, suggesting that the practical implementation of BROT (i.e., using a standard MLP with a moderate Jacobian penalty weight) can be easily done without extensive tuning.

\begin{figure}[h!]
    \centering
    \includegraphics[width=0.95\linewidth]{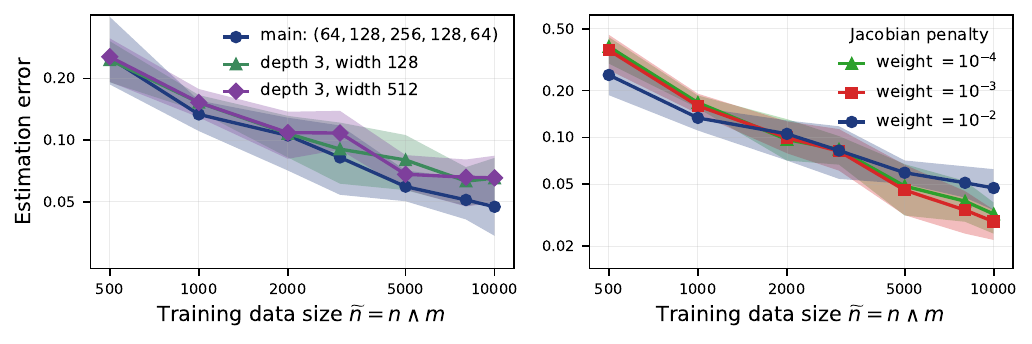}
    \caption{\textbf{Sensitivity of DNN architecture and Jacobian penalty to the statistical convergence.}
    (Left) Three MLP architectures (Jacobian penalty weight fixed to $10^{-2}$).
    (Right) Three Jacobian penalty weights (architecture fixed).}
    \label{fig:appen-arch-robustness}
\end{figure}

\subsection{Quality of Maps: Distribution Matching and Transport Cost (\cref{sec:sim})}\label{sec:appen-simulation-full}

\paragraph{Data visualization and full results.}

\cref{fig:appen-simul} visualizes the two synthetic datasets used in \cref{sec:sim}.
\cref{tab:simuls_full} reports the full results, including the standard deviations across five random trials.

\begin{figure}[h!]
    \centering
    \includegraphics[width=0.75\linewidth]{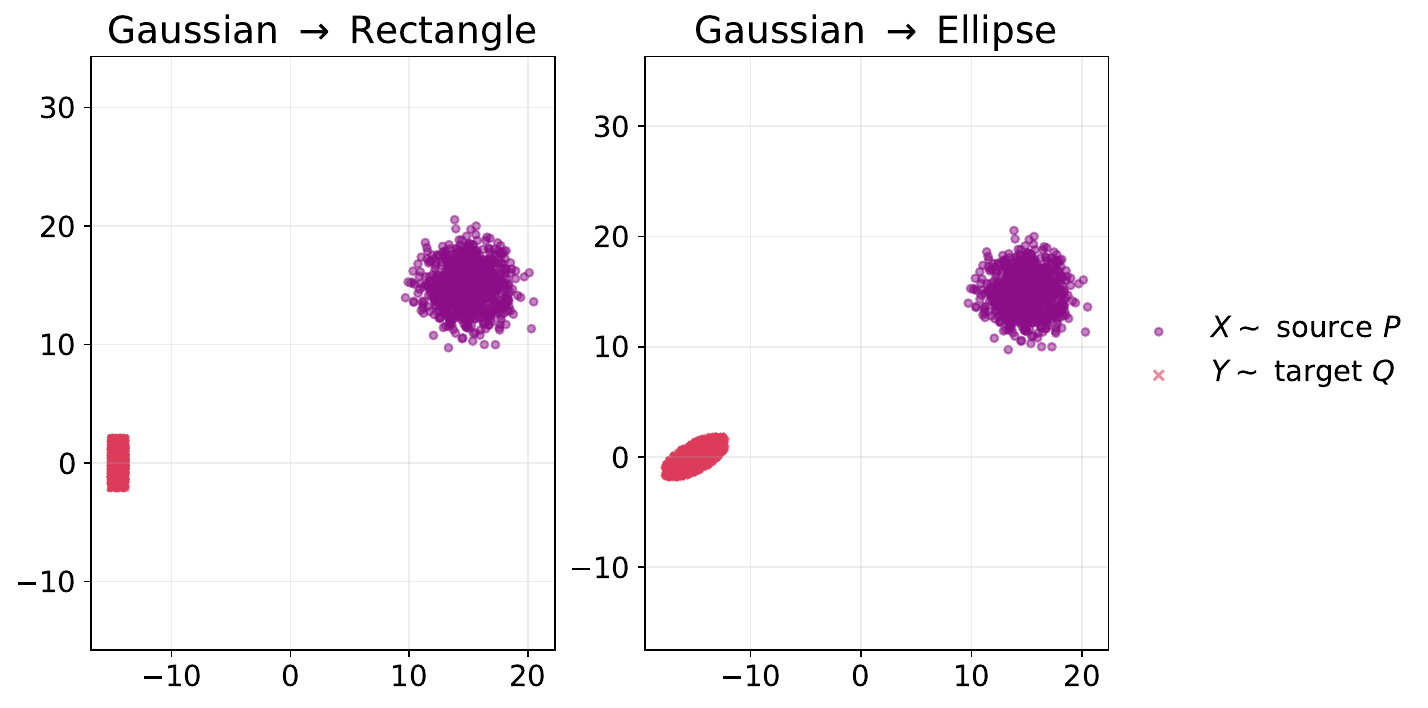}
    \caption{Source (purple) and target (red) data for the two synthetic cases used in \cref{sec:sim}.}
    \label{fig:appen-simul}
    \vskip -0.1in
\end{figure}

\begin{table}[h!]
    \centering
    \small
    \caption{\textbf{Comparison of map quality: distribution matching and transport cost.} Results are reported as mean $\pm$ standard deviation across five random trials. Best values of \texttt{Wass} are \textbf{bolded}.}
    \label{tab:simuls_full}
        \vskip 0.1in
        \begin{tabular}{llcc}
        \toprule
        Source $\to$ Target  &  Method  &  \texttt{Wass} $\downarrow$  &  \texttt{TC} $\downarrow$  \\
        \midrule
        \multirow{9}{*}{Gaussian $\to$ Rectangle} &  1NN  &  $0.126\pm0.008$  &  $1095.67\pm0.61$  \\
         &  LSOT  &  $0.111\pm0.003$  &  $1095.69\pm0.14$   \\
         &  ICNN  &  $0.211\pm0.005$  &  $1097.51\pm0.09$  \\
         &  OTP  &  $0.258\pm0.008$  &  $1099.17\pm2.39$   \\
         &  DIOTM  &  $0.123\pm0.007$  &  $1098.21\pm2.35$   \\
         & MM-v2 \cite{korotin2021neural} & $0.120\pm0.012$ & $1096.57\pm1.90$ \\
         & W2 \cite{korotin2021neural} & $0.191\pm0.008$ & $1096.44\pm1.02$ \\
         & ENT \cite{pooladian2021entropic} & $0.209\pm0.017$ & $1097.41\pm1.84$ \\

          \rowcolor{RowHighlight}
          \cellcolor{white} &  \textbf{BROT}  &  \textbf{$0.110\pm0.003$}  &  $1095.65\pm0.05$   \\
        \midrule
        \multirow{9}{*}{Gaussian $\to$ Ellipse} &  1NN  &  $0.207\pm0.075$  &  $1122.30\pm4.60$   \\
         &  LSOT  &  $0.172\pm0.036$  &  $1124.35\pm0.59$   \\
         &  ICNN  &  {$0.171\pm0.009$}  &  {$1122.54\pm1.28$}  \\
         &  OTP  &  $0.596\pm0.038$  &  $1129.88\pm7.39$   \\
         &  DIOTM  &  $0.228\pm0.020$  &  $1123.99\pm5.45$   \\
 & MM-v2 \cite{korotin2021neural} & $0.174\pm0.010$ & $1125.20\pm2.01$ \\
 & W2 \cite{korotin2021neural} & $0.163\pm0.013$ & $1124.44\pm1.61$ \\
 & ENT \cite{pooladian2021entropic} & $0.257\pm0.030$ & $1125.98\pm2.77$ \\

          \rowcolor{RowHighlight}
          \cellcolor{white} &  \textbf{BROT}  &  \textbf{$0.155\pm0.009$}  &  $1125.08\pm0.33$   \\
        \bottomrule
        \end{tabular}
\end{table}

\paragraph{Training and inference time.}

\cref{tab:simuls_2} reports the training and inference times on a single RTX 3090, averaged over five random trials and normalized by BROT ($=100\%$).
Training time includes initialization, OT computation, DNN fitting, and validation, excluding data loading and hyperparameter search.
For inference, we use batch size $128$ and measure the elapsed time on $1{,}000$ synthetic or $500$ \textsc{AFHQv2} test points.
1NN is up to $4.8\times$ slower than BROT at inference because it performs a nearest neighbor search over the training data for each test point.
LSOT is $1.5$-$2.7\times$ slower than BROT at training, while the min-max learning methods (ICNN, OTP, and DIOTM) train $8$-$50\times$ slower than BROT.

\begin{table}[h!]
    \centering
    \small
    \caption{
    \textbf{Comparison of computation time.}
    Relative training and inference times averaged over five random trials, normalized by BROT ($=100\%$).
    The synthetic results are averaged over the two synthetic datasets.
    }
    \label{tab:simuls_2}
    \vskip 0.1in
    \begin{tabular}{lcccc}
        \toprule
        Method  &  \multicolumn{2}{c}{Synthetic cases}  &  \multicolumn{2}{c}{\textsc{AFHQv2} wild $\to$ cat}  \\
        \cmidrule(lr){2-3}\cmidrule(lr){4-5}
         &  Train (\%)  &  Inference (\%)  &  Train (\%)  &  Inference (\%)  \\
        \midrule
        \midrule
        1NN \cite{manole2024pluginestimationsmoothoptimal}    &  {20\%}   &  {339\%}  &  {17\%}              &  {486\%}  \\
        LSOT \cite{seguy2018large}                            &  {267\%}  &  {95\%}     &  154\%                      &  {108\%}  \\
        ICNN \cite{pmlr-v119-makkuva20a}   &  {2129\%}  &  {87\%}              &  {4985\%}      &  {229\%}  \\
        OTP \cite{choi2025overcoming}                         &  {1078\%}  &  {105\%}             &  785\%                      &  {286\%}  \\
        DIOTM \cite{choi2025improving}                       &  {1445\%}  &  {81\%}              &  2958\%                     &  {205\%}  \\
MM-v2 \cite{korotin2021neural} & 469\% & 1658\% & 4422\% & 543\% \\
W2 \cite{korotin2021neural} & 439\% & 1602\% & 4443\% & 559\% \\
ENT \cite{pooladian2021entropic} & 15\% & 1211\% & 13\% & 1520\% \\

        \rowcolor{RowHighlight}
        \textbf{BROT}                             &  100\%  &  100\%             &  100\%                      &  {100\%}  \\
        \bottomrule
    \end{tabular}
\end{table}

\paragraph{Decoded images of transported source representation using BROT (\cref{sec:sim}).}\label{sec:appen-decoded}

\cref{fig:appen-decoded} shows training source wild animal images and decoded output images using BROT, with unpaired real target images illustrating the target cat image distribution.
We use a fixed AFHQ-fine-tuned SD-VAE decoder.
We can see that the key semantics of the source images are preserved well in the decoded images of BROT (e.g., the background colors and the fur colors).
For the quantitative comparison, we also compute FID against $2{,}000$ real target images.
BROT achieves lower FID than all the baselines, as shown in \cref{tab:appen-fid}.

\begin{figure}[h!]
    \centering
    \begin{subfigure}[t]{0.32\linewidth}
        \centering
        \setlength{\tabcolsep}{0.6pt}
        \begin{tabular}{@{}ccc@{}}
        \includegraphics[width=0.324\linewidth]{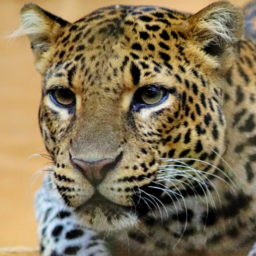} &
        \includegraphics[width=0.324\linewidth]{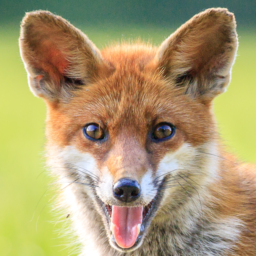} &
        \includegraphics[width=0.324\linewidth]{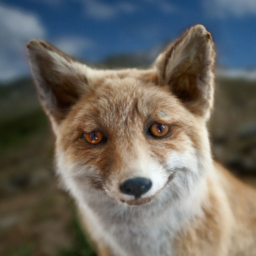} \\
        \includegraphics[width=0.324\linewidth]{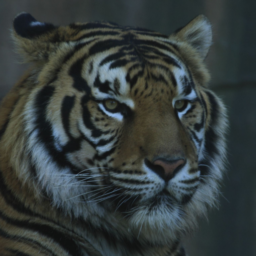} &
        \includegraphics[width=0.324\linewidth]{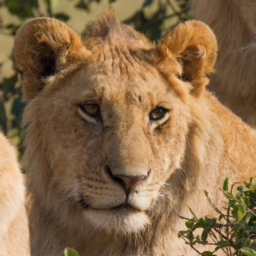} &
        \includegraphics[width=0.324\linewidth]{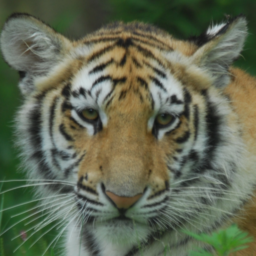} \\
        \includegraphics[width=0.324\linewidth]{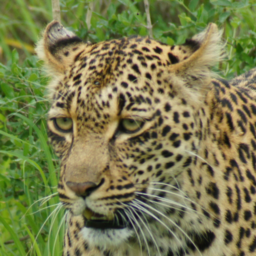} &
        \includegraphics[width=0.324\linewidth]{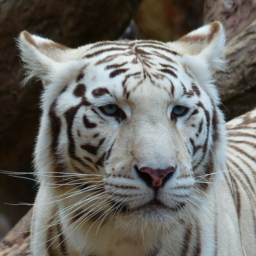} &
        \includegraphics[width=0.324\linewidth]{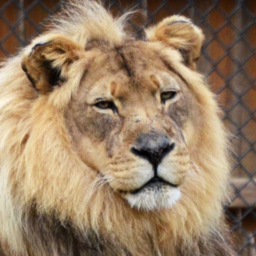}
        \end{tabular}
        \caption{Source}
        \label{fig:appen-decoded-training}
    \end{subfigure}\hfill
    \begin{subfigure}[t]{0.32\linewidth}
        \centering
        \setlength{\tabcolsep}{0.6pt}
        \begin{tabular}{@{}ccc@{}}
        \includegraphics[width=0.324\linewidth]{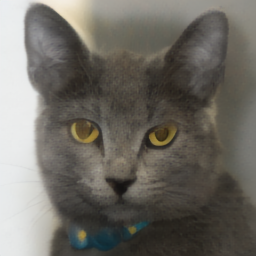} &
        \includegraphics[width=0.324\linewidth]{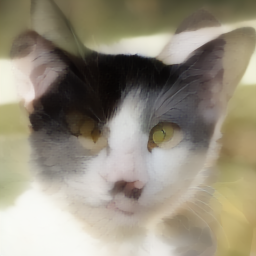} &
        \includegraphics[width=0.324\linewidth]{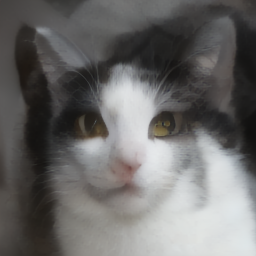} \\
        \includegraphics[width=0.324\linewidth]{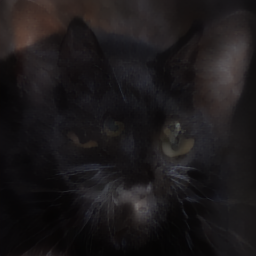} &
        \includegraphics[width=0.324\linewidth]{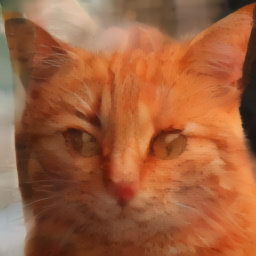} &
        \includegraphics[width=0.324\linewidth]{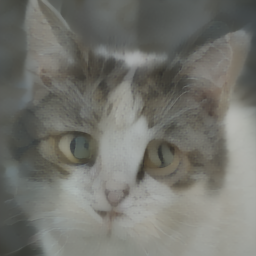} \\
        \includegraphics[width=0.324\linewidth]{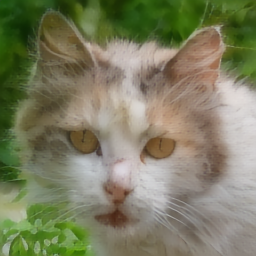} &
        \includegraphics[width=0.324\linewidth]{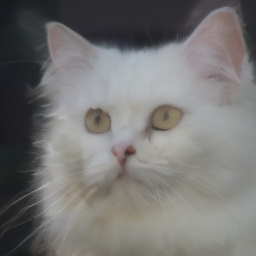} &
        \includegraphics[width=0.324\linewidth]{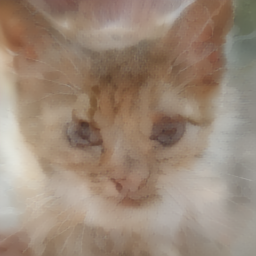}
        \end{tabular}
        \caption{BROT decoded}
        \label{fig:appen-decoded-brot}
    \end{subfigure}\hfill
    \begin{subfigure}[t]{0.32\linewidth}
        \centering
        \setlength{\tabcolsep}{0.6pt}
        \begin{tabular}{@{}ccc@{}}
        \includegraphics[width=0.324\linewidth]{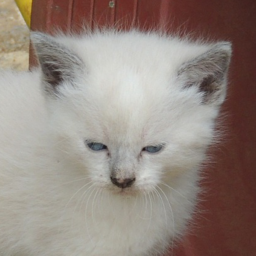} &
        \includegraphics[width=0.324\linewidth]{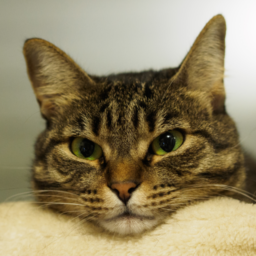} &
        \includegraphics[width=0.324\linewidth]{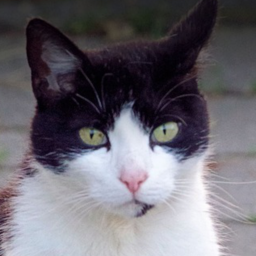} \\
        \includegraphics[width=0.324\linewidth]{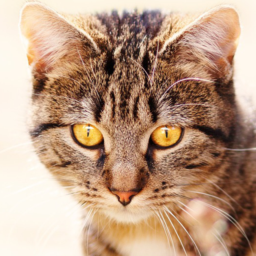} &
        \includegraphics[width=0.324\linewidth]{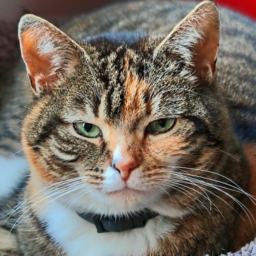} &
        \includegraphics[width=0.324\linewidth]{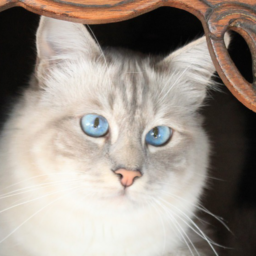} \\
        \includegraphics[width=0.324\linewidth]{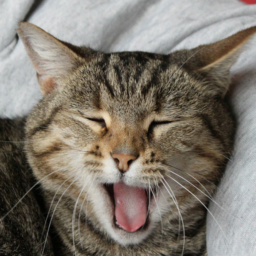} &
        \includegraphics[width=0.324\linewidth]{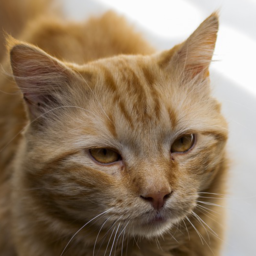} &
        \includegraphics[width=0.324\linewidth]{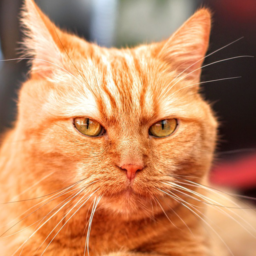}
        \end{tabular}
        \caption{Target}
        \label{fig:appen-decoded-targets}
    \end{subfigure}
    \caption{\textbf{Decoded BROT outputs on \textsc{AFHQv2} wild $\to$ cat.} (a) Source images. (b) Decoded BROT outputs corresponding by position to (a). (c) Unpaired target examples.}
    \label{fig:appen-decoded}
\end{figure}

\begin{table}[h!]
    \centering
    \small
    \caption{\textbf{Comparison of FIDs of the decoded transported images on \textsc{AFHQv2}.}
    }
    \label{tab:appen-fid}
    \vskip 0.1in

    \resizebox{\linewidth}{!}{
    \begin{tabular}{l|ccccccccc}
        \toprule
        Method
        & 1NN \cite{manole2024pluginestimationsmoothoptimal}
        & LSOT \cite{seguy2018large}
        & ICNN \cite{pmlr-v119-makkuva20a}
        & OTP \cite{choi2025overcoming}
        & DIOTM \cite{choi2025improving}
        & MM-v2 \cite{korotin2021neural}
        & W2 \cite{korotin2021neural}
        & ENT \cite{pooladian2021entropic}
        & \cellcolor{RowHighlight}\textbf{BROT} \\
        \midrule\midrule
        FID $\downarrow$
        & 22.41
        & 33.01
        & 58.73
        & 59.03
        & 32.38
        & 40.89
        & 78.11
        & 31.40
        & \cellcolor{RowHighlight}\textbf{18.67} \\
        \bottomrule
    \end{tabular}
    }
\end{table}

\subsection{Application 1: Single-Cell Perturbation Prediction (\cref{sec:app_cell})}\label{sec:appen-cell-full}

We follow the experimental setup described in \cref{sec:app_cell}.
For MM-v2 and W2, we select the architecture and learning rate using the validation MMD for the ixazomib perturbation, from widths $(128,128,64)$ or $(256,256,128)$ and learning rates $3\cdot10^{-4}$ or $10^{-3}$, with batch size $256$.
For ENT, we similarly select $\rho$ using the validation split and use the selected value $\rho=0.001$ for all drug perturbations.
\cref{tab:appen-single-cell-full} reports the full results, including the standard deviation across the $35$ drug perturbations.

 \begin{table}[h!]
    \centering
    \small
    \caption{\textbf{Comparison of single-cell perturbation prediction performance.} The results are reported as mean $\pm$ standard deviation across the $35$ drug perturbations. Among the transport estimators (excluding Identity) that satisfy $\texttt{MMD}\le0.028$, the best value is \textbf{bolded} and the second-best is \underline{underlined} for each metric.}
    \label{tab:appen-single-cell-full}
    \vskip 0.1in
    \begin{tabular}{lcccc}
        \toprule
        Method  &  \texttt{MMD} $\downarrow$  &  {$L_2$ difference $\downarrow$}  &  {$r$-std $\uparrow$}  &  {\texttt{TC} $\downarrow$}  \\
        \midrule\midrule
        Identity   &   0.028 \scriptsize{$\pm$ 0.029}   &   {1.213 \scriptsize{$\pm$ 0.899}}   &   {0.917 \scriptsize{$\pm$ 0.081}}   &  0.000 \scriptsize{$\pm$ 0.000} \\
        1NN \citep{manole2024pluginestimationsmoothoptimal}   &   \textbf{0.002} \scriptsize{$\pm$ 0.000}   &   {\textbf{0.189} \scriptsize{$\pm$ 0.062}}   &  {0.990 \scriptsize{$\pm$ 0.009}}  &  6.745 \scriptsize{$\pm$ 6.108} \\
        LSOT \citep{seguy2018large}   &   0.012 \scriptsize{$\pm$ 0.002}   &   {0.287 \scriptsize{$\pm$ 0.121}}   &   {0.982 \scriptsize{$\pm$ 0.013}}   &  4.381 \scriptsize{$\pm$ 5.437} \\
        ICNN \citep{pmlr-v119-makkuva20a}   &  {\underline{0.006}} \scriptsize{{$\pm$ 0.001}}  &   {{0.258} \scriptsize{{$\pm$ 0.096}}}   &   {{0.987} \scriptsize{{$\pm$ 0.008}}}   &  4.329 \scriptsize{$\pm$ 5.751} \\
        OTP \citep{choi2025overcoming}   &   0.035 \scriptsize{$\pm$ 0.029}   &   {1.435 \scriptsize{$\pm$ 0.852}}   &   {0.917 \scriptsize{$\pm$ 0.081}}   &  0.514 \scriptsize{$\pm$ 0.141} \\
        DIOTM \citep{choi2025improving}   &   \textbf{0.002} \scriptsize{$\pm$ 0.001}   &   {0.255 \scriptsize{$\pm$ 0.101}}   &  {0.990 \scriptsize{$\pm$ 0.010}}  &  3.993 \scriptsize{$\pm$ 5.875} \\
        MM-v2 \cite{korotin2021neural} & \textbf{0.002} \scriptsize{$\pm$ 0.001} & 0.226 \scriptsize{$\pm$ 0.078} & 0.990 \scriptsize{$\pm$ 0.006} & 3.984 \scriptsize{$\pm$ 5.821} \\
        W2 \cite{korotin2021neural} & \textbf{0.002} \scriptsize{$\pm$ 0.001} & 0.199 \scriptsize{$\pm$ 0.086} & 0.990 \scriptsize{$\pm$ 0.007} & \underline{3.934} \scriptsize{$\pm$ 4.954} \\
        ENT \cite{pooladian2021entropic} & \textbf{0.002} \scriptsize{$\pm$ 0.000} & 0.204 \scriptsize{$\pm$ 0.072} & \underline{0.992} \scriptsize{$\pm$ 0.006} & 6.291 \scriptsize{$\pm$ 6.397} \\

        \rowcolor{RowHighlight}
        \textbf{BROT}   &  {{\textbf{0.002}} \scriptsize{{$\pm$ 0.000}}}  &  {{\underline{0.194}} \scriptsize{{$\pm$ 0.076}}}  &  {{\textbf{0.993}} \scriptsize{{$\pm$ 0.005}}}  &  \textbf{3.928} \scriptsize{$\pm$ 5.356} \\
        \bottomrule
    \end{tabular}
    \vskip -0.1in
\end{table}

\subsection{Application 2: Unsupervised Domain Adaptation (\cref{sec:app_da})}\label{sec:appen-da-full}

We follow the experimental setup described in \cref{sec:app_da}.
After training each OT map, we hold out $20\%$ of the transported source training data as validation data and select the classifier iteration by the source validation accuracy.
For MM-v2 and W2, we select the network width from $(d,d,d)$ and $(2d,2d,d)$ and the learning rate from $3\cdot10^{-4}$ and $10^{-3}$, with batch size $256$.
The downstream classifier is trained using Adam with learning rate $10^{-3}$ and batch size $512$.

\cref{tab:appen-da-extended-std} reports the results as mean $\pm$ standard deviation over three random trials.
\textsc{USPS} contains $7{,}291$ source images, \textsc{MNIST} contains $60{,}000$ target training and $10{,}000$ test images, and \textsc{VisDA-17} contains $152{,}397$ synthetic source and $55{,}388$ real target images.

\begin{table}[h!]
    \centering
    \small
    \caption{
    \textbf{Comparison of unsupervised domain adaptation performance.}
    Per-class mean accuracy (\%) on the target-domain test data, reported as mean $\pm$ standard deviation over three random trials.
    }
    \label{tab:appen-da-extended-std}
    \setlength{\tabcolsep}{3.5pt}
\vskip 0.1in
\begin{tabular}{lcccc}
\toprule
  &  \multicolumn{2}{c}{\textsc{USPS} $\to$ \textsc{MNIST}}  &  \multicolumn{2}{c}{\textsc{VisDA-17} (synthetic $\to$ real)}  \\
\cmidrule(lr){2-3}\cmidrule(lr){4-5}
Method  &  CLIP-ViT/B-32  &  DINOv2-ViT/B-14  &  CLIP-ViT/B-32  &  DINOv2-ViT/B-14  \\
\midrule
    1NN \cite{manole2024pluginestimationsmoothoptimal}  &  80.34 \scriptsize{$\pm$ 0.34}  &  75.25 \scriptsize{$\pm$ 0.77}  &  74.17 \scriptsize{$\pm$ 0.08}  &  74.39 \scriptsize{$\pm$ 0.01}  \\
    LSOT \cite{seguy2018large}  &  30.12 \scriptsize{$\pm$ 2.84}  &  52.20 \scriptsize{$\pm$ 0.88}  & 76.87 \scriptsize{$\pm$ 0.63} &  \underline{81.57} \scriptsize{$\pm$ 0.18}  \\
    ICNN \cite{pmlr-v119-makkuva20a}  &  75.76 \scriptsize{$\pm$ 4.63}  &  54.70 \scriptsize{$\pm$ 21.08}  &  74.12 \scriptsize{$\pm$ 1.07}  &  77.77 \scriptsize{$\pm$ 0.41}  \\
    OTP \cite{choi2025overcoming}  &  49.98 \scriptsize{$\pm$ 4.79}  &  57.36 \scriptsize{$\pm$ 5.18}  &  76.15 \scriptsize{$\pm$ 0.42}  &  64.51 \scriptsize{$\pm$ 0.37}  \\
    DIOTM \cite{choi2025improving}  & 82.71 \scriptsize{$\pm$ 1.00} & 77.46 \scriptsize{$\pm$ 0.51} &  76.45 \scriptsize{$\pm$ 0.55}  &  27.05 \scriptsize{$\pm$ 35.31}  \\
    MM-v2 \cite{korotin2021neural} & \underline{84.00} \scriptsize{$\pm$ 1.53} & \textbf{77.73} \scriptsize{$\pm$ 5.08} & \underline{76.95} \scriptsize{$\pm$ 0.23} & 74.81 \scriptsize{$\pm$ 0.71} \\
    W2 \cite{korotin2021neural} & 80.50 \scriptsize{$\pm$ 3.34} & 75.49 \scriptsize{$\pm$ 2.63} & 74.70 \scriptsize{$\pm$ 0.81} & 70.40 \scriptsize{$\pm$ 0.57} \\
    ENT \cite{pooladian2021entropic} & 79.76 \scriptsize{$\pm$ 4.13} & 76.21 \scriptsize{$\pm$ 0.52} & 73.88 \scriptsize{$\pm$ 0.30} & 74.28 \scriptsize{$\pm$ 0.30} \\

    \rowcolor{RowHighlight}
    \textbf{BROT}  &  \textbf{84.83} \scriptsize{$\pm$ 1.13}  & \underline{77.60} \scriptsize{$\pm$ 0.24} &  \textbf{79.34} \scriptsize{$\pm$ 0.07}  &  \textbf{81.59} \scriptsize{$\pm$ 0.30}  \\
\bottomrule
\end{tabular}
\end{table}

\paragraph{Design of classifier head.}

We use a linear classifier for \textsc{VisDA-17} and an MLP classifier for \textsc{USPS}$\to$\textsc{MNIST}.
This choice is based on the source-validation performance of the Identity baseline: the pretrained representations are already largely linearly separable on \textsc{VisDA-17}, whereas a nonlinear classifier performs substantially better on \textsc{USPS}$\to$\textsc{MNIST}.

\paragraph{Comparison with end-to-end OT-based domain adaptation methods.}

End-to-end OT-based domain adaptation methods such as DeepJDOT \cite{damodaran2018deepjdot}, JUMBOT \cite{fatras2021jumbot}, and ELOT \cite{yang2026elastic} are more tailored for domain adaptation than BROT.
They jointly train the representation encoder, the classifier head, and the transport plan.
In contrast, BROT estimates the OT map solely from the Euclidean distances between source and target data, and does not train the encoder.
To see whether BROT can nevertheless achieve competitive performance, we compare it against these three end-to-end methods.

\begin{table}[h!]
  \centering
  \small
  \caption{\textbf{Comparison of domain adaptation performance for BROT and end-to-end OT-based domain adaptation methods.} The performance is evaluated by the per-class mean accuracy (\%) on the target-domain test data, following \citep{peng2017visdavisualdomainadaptation}. The results of BROT are averaged over three random seeds, while baseline results are copied from the original papers.}
  \label{tab:appen-da-end2end}
  \vskip 0.1in
  \begin{tabular}{lcccc}
    \toprule
    \multirow{2}{*}{Method}  &  \multicolumn{2}{c}{Train}  &  \multirow{2}{*}{Backbone}  &  \multirow{2}{*}{Accuracy}  \\
    \cmidrule(lr){2-3}
     &  Encoder  &  Head  &   &   \\
    \midrule
    DeepJDOT \cite{damodaran2018deepjdot}  &  \cmark  &  \cmark  &  ResNet-50  &  68.00  \\
    JUMBOT   \cite{fatras2021jumbot}       &  \cmark  &  \cmark  &  ResNet-50  &  72.50  \\
    ELOT     \cite{yang2026elastic}        &  \cmark  &  \cmark  &  ResNet-50  &  76.32  \\
    \midrule

       \rowcolor{RowHighlight}
       &  \xmark  &  \cmark  &  CLIP-ViT/B-32    &  79.34  \\

       \rowcolor{RowHighlight}
       &  \xmark  &  \cmark  &  DINOv2-ViT/B-14  &  81.59  \\

    \rowcolor{RowHighlight}
    \multirow{-3}{*}{\textbf{BROT}}
       &  \xmark  &  \cmark  &  ResNet-50        &  70.55  \\
    \bottomrule
  \end{tabular}
\end{table}

\cref{tab:appen-da-end2end} summarizes the results.
BROT attains a per-class mean accuracy of $79.34$\% with CLIP-ViT/B-32 and $81.59$\% with DINOv2-ViT/B-14, both higher than the end-to-end methods (DeepJDOT $68.0$, JUMBOT $72.5$, ELOT $76.32$).
The results suggest that, with strong frozen visual encoders, BROT can be competitive even without encoder fine-tuning or a task-specific training objective.

For a more fair comparison, we also report the performance of BROT using the fixed representations obtained from pretrained ResNet-50.
BROT attains $70.55$\%, which exceeds DeepJDOT but falls below JUMBOT and ELOT, as expected since BROT does not jointly train the encoder.
This gap suggests that the joint training can modestly improve domain adaptation performance, however, the relatively small gap indicates that BROT remains competitive even under the same encoder.

\subsection{Computational Complexity}\label{sec:appen-complexity}

Using full data, we compute a single (full) empirical OT plan from all source and target data.
For $n=m,$ Step 1 of BROT using the full data has $O(n^{2}d)$ complexity for building the cost matrix, $O(n^{3})$ complexity for solving the linear program, and $O(n^{2})$ peak memory \cite{flamary2021pot}.

The chunked procedure used in large-scale experiments, which is briefly introduced in \cref{sec:appen-simulation-tuning}, uses $O(n/b)$ chunk pairs of size $b$ when $n=m$, so the complexities of the cost matrix construction and the OT solver become $O(nbd)$ and $O(nb^{2})$, respectively, and the peak memory becomes $O(b^{2})$.
Thus, a large ambient dimension $d$ mainly affects the cost-matrix construction, whereas the chunk size $b$ determines the OT solver time and peak memory of the chunked procedure.
Step 2 costs $O(nN_{\textup{param}})$ per epoch, where $N_{\textup{param}}$ is the number of DNN parameters, and inference costs $O(N_{\textup{param}})$ per query.

\cref{tab:appen-complexity} compares these computational costs with those of the existing statistically optimal estimators.
It is also noteworthy that the cost of the exact OT computation is not specific to BROT, since the minimax optimal 1NN estimator computes the same OT plan.
Further, the two methods differ at inference: 1NN stores the entire training data and requires an $O(nd)$ search per query, whereas BROT requires only a single forward pass, which makes 1NN up to $4.86$ times slower than BROT at inference (see \cref{tab:simuls_2}).

\begin{table}[h!]
    \centering
    \small
    \caption{\textbf{Computational complexity of OT map estimators (up to logarithmic factors).}
    $N_{\textup{param}}$ is the number of DNN parameters and $b$ is the chunk size in chunked OT, with $n=m.$}
    \label{tab:appen-complexity}
    \vskip 0.1in
    \begin{tabular}{lp{0.66\linewidth}}
        \toprule
        Method & Computation \\
        \midrule\midrule
        Wavelet \cite{hutter2020minimaxestimationsmoothoptimal} & Optimizes the semi-dual objective over a truncated wavelet expansion whose number of coefficients grows exponentially in $d.$ \\
        1NN \cite{manole2024pluginestimationsmoothoptimal} & OT plan: $O(n^{3}).$ Inference: $O(nd)$ per query. \\
        Kernel \cite{vacher2024optimal} & Worst-case $O(n^{5}).$ \\

        \rowcolor{RowHighlight}
        \textbf{BROT} (full) & OT plan: $O(n^{3}).$ DNN training: $O(nN_{\textup{param}})$ per epoch. Inference: $O(N_{\textup{param}})$ per query. \\

        \rowcolor{RowHighlight}
        \textbf{BROT} (chunked) & OT plan: $O(nb^{2}).$ DNN training and inference: same as above. \\
        \bottomrule
    \end{tabular}
\end{table}

\paragraph{Effect of the chunk size.}

We measure the computation time of the OT plan in the setting of \cref{sec:rate} with $n=m$ and chunk size $b=512.$
As shown in \cref{tab:appen-step1-time}, the computation time of full OT grows by a factor of about $187$ as $n$ grows from $1{,}000$ to $10{,}000,$ whereas that of chunked OT grows only by a factor of about $10,$ following the linear $O(nb^{2})$ scaling.

We vary the chunk size $b$ in the Gaussian $\to$ rectangle experiment ($n=3{,}000$) and the \textsc{AFHQv2} experiment ($d=4{,}096$) of \cref{sec:sim}, and compare the results with those of BROT using full OT.
As shown in \cref{tab:appen-chunk-quality}, even a moderate chunk size attains \texttt{Wass} and \texttt{TC} close to those obtained with full OT on both datasets.
We expect that increasing the chunk size with $n$, e.g., $b=n^\kappa$ for $0<\kappa<1$, would further reduce the gap from full OT while keeping the OT computation at $O(n^{1+2\kappa})$.
A theoretical analysis of this regime is left for future work.

\begin{table}[h!]
    \centering
    \small

    \begin{minipage}[t]{0.47\linewidth}
    \centering
    \captionof{table}{\textbf{OT computation time of Step 1.} Time (seconds) in the setting of \cref{sec:rate}, with chunk size $b=512$ for chunked OT.}
    \label{tab:appen-step1-time}
    \vskip 0.1in
    \begin{tabular}{cccc}
        \toprule
        $n=m$ & Full & Chunked & Speedup \\
        \midrule\midrule
        $1{,}000$ & 0.184 & 0.086 & $2.1\times$ \\
        $2{,}000$ & 0.863 & 0.173 & $5.0\times$ \\
        $5{,}000$ & 6.698 & 0.425 & $15.8\times$ \\
        $10{,}000$ & 34.427 & 0.885 & $38.9\times$ \\
        \bottomrule
    \end{tabular}
    \end{minipage}\hfill
    \begin{minipage}[t]{0.49\linewidth}
    \centering
    \captionof{table}{
    \textbf{Map quality vs. Chunk size.}
    The results from the experiments in \cref{sec:sim} with various chunk sizes $b \in \{ 512, 1{,}024, 2{,}048 \}.$
    The full results are in \cref{tab:distmatch}.
    }
    \label{tab:appen-chunk-quality}
    \vskip 0.1in
    \begin{tabular}{ccccc}
        \toprule
        & \multicolumn{2}{c}{Rectangle} & \multicolumn{2}{c}{\textsc{AFHQv2}} \\
        \cmidrule(lr){2-3}\cmidrule(lr){4-5}
        $b$ & \texttt{Wass} & \texttt{TC} & \texttt{Wass} & \texttt{TC} \\
        \midrule\midrule
        $512$ & 0.110 & 1097.00 & 45.490 & 2000.92 \\
        $1{,}024$ & 0.107 & 1096.95 & 45.770 & 2009.81 \\
        $2{,}048$ & 0.106 & 1097.19 & 45.550 & 1993.18 \\
        \bottomrule
    \end{tabular}
    \end{minipage}
\end{table}

\clearpage
\paragraph{Entropic regularization: computation and inference cost.}\label{sec:appen-entropic-cost}

We further compare BROT with the entropic map estimator (ENT) \cite{pooladian2021entropic} in the ground-truth-map setting of \cref{sec:rate}.
For ENT, we report the result with the lowest estimation error over the grid
$\varepsilon\in\{0.02,0.05,0.1,0.2,0.3,0.5,1.0\}$, which gives a favorable comparison to ENT.

As shown in \cref{tab:appen-entropic}, the OT computation times of BROT and ENT are comparable at $n=5{,}000$, while BROT becomes faster as the sample size increases and achieves a lower estimation error at every sample size.
BROT is also faster at inference, since ENT is nonparametric and its inference cost grows with the number of target samples, whereas BROT requires only a DNN forward pass.

\cref{tab:appen-entropic-eps} further shows that ENT is sensitive to the choice of $\varepsilon$.
Reducing $\varepsilon$ below approximately $0.1$ does not improve the estimation error in this experiment, while substantially increasing the computation time.
Moreover, the appropriate value of $\varepsilon$ is not known in advance: evaluating the full grid takes $448$ seconds, compared with $78.2$ seconds for the entire BROT procedure.
Thus, accounting for hyperparameter selection can further increase the practical computational cost of ENT.

\begin{table}[h!]
    \centering
    \small

    \begin{minipage}[t]{0.5\linewidth}
    \centering
    \captionof{table}{\textbf{Comparison between BROT and ENT \cite{pooladian2021entropic} in terms of computation time.}
    (Upper) OT computation time in seconds, with estimation error in parentheses.
    (Lower) Elapsed to transport $20{,}000$ test data points, in milliseconds.
    For ENT, the best $\varepsilon$ over the evaluated grid is reported.}
    \label{tab:appen-entropic}
    \vskip 0.1in
    \begin{tabular}{ccc}
        \toprule
        $\widetilde{n} = m$ & ENT & \cellcolor{RowHighlight}\textbf{BROT} \\
        \midrule\midrule
        \multicolumn{3}{l}{\textit{OT computation: seconds (estimation error)}} \\
        $5{,}000$ & 4.38 (0.075) & \cellcolor{RowHighlight}\textbf{5.19 (0.066)} \\
        $8{,}000$ & 11.52 (0.057) & \cellcolor{RowHighlight}\textbf{9.06 (0.052)} \\
        $10{,}000$ & 16.82 (0.055) & \cellcolor{RowHighlight}\textbf{10.99 (0.047)} \\
        \midrule
        \multicolumn{3}{l}{\textit{Inference: milliseconds}} \\
        $5{,}000$ & 9.83 & \cellcolor{RowHighlight}\textbf{6.01} \\
        $8{,}000$ & 15.19 & \cellcolor{RowHighlight}\textbf{5.79} \\
        $10{,}000$ & 18.75 & \cellcolor{RowHighlight}\textbf{6.10} \\
        \bottomrule
    \end{tabular}
    \end{minipage}\hfill
    \begin{minipage}[t]{0.46\linewidth}
    \centering
    \captionof{table}{
    \textbf{Sensitivity of ENT to the regularization strength $\varepsilon$.}
    Computation time and estimation error at $\widetilde n=10{,}000$.
    The total time and error of BROT are shown for reference.
    }
    \label{tab:appen-entropic-eps}
    \vskip 0.1in
    \begin{tabular}{lcc}
        \toprule
        $\varepsilon$ & Time (s) & Error \\
        \midrule\midrule
        $1.0$ & 4.68 & 0.348 \\
        $0.5$ & 6.23 & 0.113 \\
        $0.3$ & 8.26 & 0.078 \\
        $0.2$ & 10.59 & 0.064 \\
        $0.1$ & 16.82 & 0.055 \\
        $0.05$ & 28.71 & 0.056 \\
        $0.02$ & 372.75 & 0.059 \\
        Sweep over the grid & 448.00 & -- \\
        \midrule

        \rowcolor{RowHighlight}
        \textbf{BROT} (total) & \textbf{78.20} & \textbf{0.047} \\
        \bottomrule
    \end{tabular}
    \end{minipage}
\end{table}

\subsection{Robustness to Input Perturbations}\label{sec:appen-robustness}

The DNN class in \cref{thm:dnn_main} is Lipschitz-constrained, so $\Vert \widehat{\mathbf{T}}(x+\delta) - \widehat{\mathbf{T}}(x) \Vert \le (\lambda + c') \Vert \delta \Vert$ for any input perturbation $\delta$ with $x,x+\delta\in\Omega.$
That is, the output difference grows at most linearly with the perturbation size.
We also examine this property empirically in the experiment of \cref{sec:rate} with $n=5{,}000$ and the Jacobian penalty weight $10^{-2},$ applying random $L_{2}$ perturbations and a $20$-step PGD attack \cite{madry2018towards}.
PGD maximizes $\|\widehat{\mathbf T}(x+\delta)-\widehat{\mathbf T}(x)\|_2$ subject to $\|\delta\|_2\le\varepsilon_{\textup{pert}}$.
\cref{tab:appen-robustness} reports the mean output difference, which grows approximately linearly with $\varepsilon_{\textup{pert}}$.
The largest observed output-to-input displacement ratio decreases from $11.4$-$18.5$ without the penalty to $5.7$-$5.8$ with penalty weight $10^{-2}$, showing that the Jacobian penalty makes the trained map less sensitive to both random and adversarial perturbations.

\begin{table}[h!]
    \centering
    \small

    \caption{\textbf{Average output difference under input perturbations.} Results in the setting of \cref{sec:rate} with $n=5{,}000$.}
    \label{tab:appen-robustness}
    \vskip 0.1in
    \begin{tabular}{ccc}
        \toprule
        $\varepsilon_{\textup{pert}}$ & Random $L_{2}$ & $20$-step PGD \\
        \midrule\midrule
        $0.05$ & 0.089 & 0.118 \\
        $0.10$ & 0.179 & 0.237 \\
        $0.20$ & 0.357 & 0.471 \\
        $0.50$ & 0.897 & 1.180 \\
        \bottomrule
    \end{tabular}
\end{table}

\clearpage
\section{Background on Existing Estimators}\label{sec:baselines_details}

\paragraph{Wavelet-based estimator.}

\cite{hutter2020minimaxestimationsmoothoptimal} proposed an estimator for smooth OT maps based on a wavelet approximation of the potential function.
They considered the empirical dual objective
$$
    S_n(\phi) := \int \phi(x)  dP_n(x) + \int \phi^{*}(y)  dQ_n(y),
$$
where $\phi^{*}(y) := \sup_{x \in \Omega} \{\langle x,y\rangle - \phi(x)\}$ denotes the convex conjugate of $\phi$.
They then restricted $\phi$ to a finite-dimensional wavelet class $ \mathcal{F}^{\textup{W}}_J := \{ \phi_w(x) = \sum_{j=1}^{J} w_j \psi_j(x) : w = [w_{1}, \ldots, w_{J}]^{\top} \in \mathbb{R}^{J} \}, $
where $\psi_j, j \in [J]$ are suitable wavelet basis functions.
The wavelet-based estimator is then defined as
$ \widehat{\phi}^{\textup{W}} \in \argmin_{\phi \in \mathcal{F}^{\textup{W}}_J} S_n(\phi), \widehat{\mathbf T}^{\textup{W}}_{nm} := \nabla \widehat{\phi}^{\textup{W}}. $
As a result, $\widehat{\mathbf T}^{\textup{W}}_{nm}$ attains the minimax optimal convergence rate.
However, selecting the basis functions $\psi_{j}, j \in [J]$, as well as the number of basis functions $J$, and computing the coefficients $w_{j}, j \in [J]$, requires solving a high-dimensional constrained optimization problem whose complexity grows quickly with the dimension, which severely limits its applicability, as discussed in \cite{hutter2020minimaxestimationsmoothoptimal} and \cite{manole2024pluginestimationsmoothoptimal}.

\paragraph{Nearest neighbor estimator.}

As a complementary line of work, \cite{manole2024pluginestimationsmoothoptimal} introduced an estimator for Lipschitz OT maps that also attains the minimax optimal rate.
Based on the barycentric map $\mathbf T_{\hat\Gamma}$, which is defined only on the observed source data $X_i, i \in [n],$ they extended it to all $x \in \Omega$ using Voronoi cells $V_{i}, i \in [n]$:
$$
\widehat{\mathbf T}^{1\textup{NN}}_{nm}(x)
= \sum_{i=1}^{n} \mathbf T_{\hat\Gamma}(X_i)   \mathbb{I} \left( x \in V_i \right),
$$
where
$ V_i:=\big\{x\in\Omega:\ \|x-X_i\|\le \|x-X_j\|,\ \forall j\neq i\big\}. $
As a result, $\widehat{\mathbf T}^{1\textup{NN}}_{nm}$ achieves the optimal convergence rate when $\mathbf T_0$ is Lipschitz.
On the downside, inference at a new point $x$ requires a nearest neighbor search among the $n$ training data, so both memory and latency scale with $n$, which can be prohibitive for large-scale applications.

\paragraph{Estimators in general classes of functions.}

More recently, \cite{10.1214/24-AOS2482} proposed a general estimation framework for rich classes of functions of convex potentials.
Similar to the wavelet-based estimator, they considered $ \widehat{\phi}_{\mathcal{F}} \in \argmin_{\phi\in \mathcal{F}} S_{n}(\phi)$ and define $ \widehat{\mathbf{T}}_{\mathcal{F}} := \nabla\widehat{\phi}_{\mathcal{F}}, $ where $\mathcal{F}$ is a given candidate class of potentials.
They also derived oracle inequalities showing that the risk $\mathbb{E}\Vert \widehat{\mathbf{T}}_{\mathcal{F}} - \mathbf{T}_0 \Vert_{L^2(P)}^{2}$ decomposes into an approximation term and a complexity term governed by the metric entropy of $\mathcal{F}$, and they showed that, for a suitably chosen class $\mathcal{F}$ such as ReQU DNNs, the resulting estimator attains the minimax optimal rate.

From a computational viewpoint, however, this plug-in strategy remains challenging.
Evaluating the dual objective $S_n(\phi)$ requires computing the conjugate $\phi^{*}(y) = \sup_{x} \{\langle x,y\rangle - \phi(x)\}$ at many target points $y.$
When $\phi$ is parameterized by a neural network, each conjugate evaluation amounts to solving a high-dimensional inner optimization problem over $x$, which is nonconvex in general and as hard as global optimization.
Input convex neural networks (ICNNs) \cite{pmlr-v119-makkuva20a} offer a partial algorithmic remedy: by designing the architecture so that $\phi$ is convex, the inner problem defining $\phi^{*}(y)$ becomes a convex optimization and can be approximately solved by standard first-order methods.
However, despite the practical implementation, the statistical optimality of ICNN-based OT estimators remains unclear, since the approximation capabilities of ICNNs (i.e., their ability to approximate smooth convex functions) are not yet well-investigated, which is a promising topic.

\paragraph{Neural OT estimators.}

\cite{tarasov2026statistical} studied semi-dual min-max OT with separately parameterized potentials and maps.
For the exact empirical saddle point $\widehat{\mathbf T}$, they derived a bound of the form
$$
\mathbb E\|\widehat{\mathbf T}-\mathbf T_0\|_{L^2(P)}^2
\le \tau+O(n^{-1/2})+O(m^{-1/2}),
$$
where the network class is fixed in advance to achieve a prescribed approximation error $\tau>0$.
For a fixed class, however, $\tau$ does not decrease with the sample size, so this bound does not directly provide a vanishing minimax rate.
Moreover, decreasing $\tau$ generally requires a richer network class, and the dependence of the statistical terms on this class must also be controlled.

To establish a minimax rate, one would therefore need to let the network class grow with the sample size, control the resulting approximation and statistical errors jointly, and compare the resulting rate with the minimax lower bound.
For BROT, we explicitly construct a sample size-dependent DNN class whose approximation error vanishes at the required rate, and \cref{thm:dnn_main} shows that the resulting estimation rate matches the minimax lower bound up to logarithmic factors.
The minimax optimality of semi-dual neural OT estimators remains an interesting open problem.

\paragraph{Barycentric estimators.}

We here provide detailed comparison between BROT and existing barycentric estimators that use the barycentric targets when estimating the OT map.
\cref{rmk:brot_vs_lsot} compares BROT with LSOT from \cite{seguy2018large}, while \cref{rmk:brot_vs_perrot} compares BROT with a joint estimation of the coupling and map from \cite{perrot2016mapping}.

\begin{remark}[BROT vs. LSOT]\label{rmk:brot_vs_lsot}
    We use the unregularized OT plan in BROT despite the computational benefits often associated with regularized plans such as LSOT, for the following reasons.
    First, our minimax optimal analysis in \cref{thm:dnn_main} applies to the unregularized OT estimator, whereas the minimax optimality for the regularized DNN estimators is unknown.
    Second, regularization introduces bias relative to the unregularized ground-truth OT map, which is empirically shown in \cref{tab:distmatch}.
    That is, BROT achieves lower $\texttt{Wass}$ than LSOT with a comparable \texttt{TC}.
    Third, the runtime efficiency of using the regularized OT plan is not substantial in our experiments.
    \cref{tab:simuls_2} shows that LSOT trains comparably to or sometimes slower than BROT in our experiments.
    This is because LSOT sometimes requires many iterations to learn the dual potentials by stochastic ascent until convergence, while BROT solves a single linear program for the first step.
    The training cost and inference cost of the second step are almost identical for the two methods, when they use the same architectures.
    Hence, the unregularized objective does not lead to a substantial runtime disadvantage.
\end{remark}

\begin{remark}[BROT vs. joint coupling--map estimation]\label{rmk:brot_vs_perrot}

\cite{perrot2016mapping} also considered learning an out-of-sample map from barycentric targets.
Their method jointly estimates the coupling and a linear or kernel map by alternating optimization, so the regression objective also affects the coupling.
In contrast, BROT first computes the unregularized empirical OT plan and then fits a DNN to its barycentric targets.
Their analysis provides an error decomposition, but does not establish a sample-size-dependent convergence rate, hence its minimax optimality is unknown.
In contrast, \cref{thm:dnn_main} shows that the DNN estimator of BROT attains the minimax rate up to logarithmic factors.

\end{remark}

\paragraph{Kernel-based estimator.}

\cite{vacher2024optimal} proposed a statistically optimal estimator based on a kernel sum-of-squares formulation.
Their method estimates a scalar transport potential and obtains the OT map by differentiating the estimated potential, whereas BROT directly fits the transport map to empirical barycentric targets.
However, their exact estimator is difficult to implement in practice, and also has a worst-case computational cost of $O(n^5)$ up to a logarithmic factor, while the empirical OT step of BROT costs $O(n^3)$ for full OT and $O(nb^2)$ with chunking.

\section{Broader Impacts}\label{sec:appen-broader}

This work is a methodological contribution to optimal transport map estimation and does not propose a deployed system or release a generative model intended for direct societal use.
The estimator can in principle be applied within downstream pipelines such as domain adaptation and biological perturbation prediction, where the broader impact would be inherited from the surrounding application rather than from the OT estimator itself.
We therefore do not foresee specific negative societal effects beyond those already associated with the application areas in which OT methods are deployed.

\clearpage
\section*{NeurIPS Paper Checklist}

\begin{enumerate}

\item {\bf Claims}\label{checklist:claims}
    \item[] Question: Do the main claims made in the abstract and introduction accurately reflect the paper's contributions and scope?
    \item[] Answer: \answerYes{}
    \item[] Justification: The abstract and introduction in \cref{sec:intro} state our proposed framework, theoretical properties, and empirical results.
    The main claims are supported by \cref{sec:minimax,sec:exp_cdot}.

\item {\bf Limitations}\label{checklist:limitations}
    \item[] Question: Does the paper discuss the limitations of the work performed by the authors?
    \item[] Answer: \answerYes{}
    \item[] Justification: The discussion in \cref{sec:discussion} states limitations of this work and future directions. \cref{sec:appen-sensitivity} also clarifies the gap between the theoretical architecture and the practical implementation of DNN.

\item {\bf Theory assumptions and proofs}\label{checklist:theory}
    \item[] Question: For each theoretical result, does the paper provide the full set of assumptions and a complete (and correct) proof?
    \item[] Answer: \answerYes{}
    \item[] Justification: The assumptions are stated in \cref{sec:minimax,sec:appen-notation}, and the proof of \cref{thm:dnn_main} is given in \cref{sec:appen-full_proof}.

    \item {\bf Experimental result reproducibility}\label{checklist:reproducibility}
    \item[] Question: Does the paper fully disclose all the information needed to reproduce the main experimental results of the paper to the extent that it affects the main claims and/or conclusions of the paper (regardless of whether the code and data are provided or not)?
    \item[] Answer: \answerYes{}
    \item[] Justification: The main experimental setup is described in \cref{sec:exp_cdot}. Dataset construction, model architecture choices, and hyperparameter search grids are given in \cref{sec:appen-exps}.
    We also provide source code as supplementary material.

\item {\bf Open access to data and code}\label{checklist:data-code}
    \item[] Question: Does the paper provide open access to the data and code, with sufficient instructions to faithfully reproduce the main experimental results, as described in supplemental material?
    \item[] Answer: \answerYes{}
    \item[] Justification: On the first page, we provide a link to the public repository containing our source code.
    The datasets used in the paper are public and are cited in \cref{sec:exp_cdot}.
    Implementation details are given in \cref{sec:appen-exps}.

\item {\bf Experimental setting/details}\label{checklist:experimental-details}
    \item[] Question: Does the paper specify all the training and test details (e.g., data splits, hyperparameters, how they were chosen, type of optimizer) necessary to understand the results?
    \item[] Answer: \answerYes{}
    \item[] Justification: Training/test splits, model architectures, optimizers, hyperparameter grids, and validation criteria are described in \cref{sec:exp_cdot,sec:appen-simulation-tuning}.
    Experiment-specific settings are given in \cref{sec:appen-convergence_nonlinear,sec:appen-simulation-full,sec:appen-cell-full,sec:appen-da-full}; \cref{alg:mlp-ot} includes the practical Jacobian-penalized training loss.

\item {\bf Experiment statistical significance}\label{checklist:uncertainty}
    \item[] Question: Does the paper report error bars suitably and correctly defined or other appropriate information about the statistical significance of the experiments?
    \item[] Answer: \answerYes{}
    \item[] Justification:
    Standard deviations for map quality, single-cell prediction, and domain adaptation are reported in
    \cref{tab:simuls_full,tab:appen-single-cell-full,tab:appen-da-extended-std}.
    \cref{fig:appen-convergence-sem} reports one-standard-error bands for the convergence experiment.

\item {\bf Experiments compute resources}\label{checklist:compute}
    \item[] Question: For each experiment, does the paper provide sufficient information on the computer resources (type of compute workers, memory, time of execution) needed to reproduce the experiments?
    \item[] Answer: \answerYes{}
    \item[] Justification: Experiments use NVIDIA RTX 3090 GPUs, with one GPU per process for the timing comparisons.
    Common implementation and per-process timing settings are stated in \cref{sec:appen-simulation-tuning,sec:appen-simulation-full}; runtimes and complexity are reported in \cref{tab:simuls_2,sec:appen-complexity}.

\item {\bf Code of ethics}\label{checklist:ethics}
    \item[] Question: Does the research conducted in the paper conform, in every respect, with the NeurIPS Code of Ethics \url{https://neurips.cc/public/EthicsGuidelines}?
    \item[] Answer: \answerYes{}
    \item[] Justification: This work uses public datasets and standard benchmarks, and does not involve data collection from human subjects or release of high-risk models.

\item {\bf Broader impacts}\label{checklist:impacts}
    \item[] Question: Does the paper discuss both potential positive societal impacts and negative societal impacts of the work performed?
    \item[] Answer: \answerYes{}
    \item[] Justification: A short discussion of broader impacts is given in \cref{sec:appen-broader}. This work is a methodological contribution to OT map estimation and does not deploy or release a generative model intended for direct societal use.

\item {\bf Safeguards}\label{checklist:safeguards}
    \item[] Question: Does the paper describe safeguards that have been put in place for responsible release of data or models that have a high risk for misuse (e.g., pre-trained language models, image generators, or scraped datasets)?
    \item[] Answer: \answerNA{}
    \item[] Justification: The paper does not release a high-risk pretrained model, image generator, or language model.

\item {\bf Licenses for existing assets}\label{checklist:licenses}
    \item[] Question: Are the creators or original owners of assets (e.g., code, data, models), used in the paper, properly credited and are the license and terms of use explicitly mentioned and properly respected?
    \item[] Answer: \answerYes{}
    \item[] Justification: The datasets, software packages, and pretrained encoders are cited in \cref{sec:exp_cdot,sec:appen-simulation-tuning} and the experiment-specific appendices.

\item {\bf New assets}\label{checklist:new-assets}
    \item[] Question: Are new assets introduced in the paper well documented and is the documentation provided alongside the assets?
    \item[] Answer: \answerYes{}
    \item[] Justification: The code is publicly available at the link on the first page.
    Implementation and reproduction details are described in \cref{sec:appen-simulation-tuning,sec:appen-exps}.

\item {\bf Crowdsourcing and research with human subjects}\label{checklist:human-subjects}
    \item[] Question: For crowdsourcing experiments and research with human subjects, does the paper include the full text of instructions given to participants and screenshots, if applicable, as well as details about compensation (if any)?
    \item[] Answer: \answerNA{}
    \item[] Justification: This paper does not include crowdsourcing experiments or conducted human-subject studies.

\item {\bf Institutional review board (IRB) approvals or equivalent for research with human subjects}\label{checklist:irb}
    \item[] Question: Does the paper describe potential risks incurred by study participants, whether such risks were disclosed to the subjects, and whether Institutional Review Board (IRB) approvals (or an equivalent approval/review based on the requirements of your country or institution) were obtained?
    \item[] Answer: \answerNA{}
    \item[] Justification: This paper does not include crowdsourcing experiments or newly conducted human-subject studies, so we believe that IRB approval is not applicable.

\item {\bf Declaration of LLM usage}\label{checklist:llm}
    \item[] Question: Does the paper describe the usage of LLMs if it is an important, original, or non-standard component of the core methods in this research? Note that if the LLM is used only for writing, editing, or formatting purposes and does \emph{not} impact the core methodology, scientific rigor, or originality of the research, declaration is not required.
    \item[] Answer: \answerNA{}
    \item[] Justification: The core method development and experiments do not use LLMs as an important, original, or non-standard component.

\end{enumerate}

\end{document}